\documentclass[]{Spark_Diffusion_tech_report}

\graphicspath{{figures/}}
\usepackage{amsmath,amsfonts,bm}

\def\eqref#1{equation~\ref{#1}}

\def\1{\bm{1}}

\DeclareMathAlphabet{\mathsfit}{\encodingdefault}{\sfdefault}{m}{sl}
\SetMathAlphabet{\mathsfit}{bold}{\encodingdefault}{\sfdefault}{bx}{n}

\usepackage{url}
\usepackage{wrapfig}
\usepackage{nicefrac}
\usepackage{multicol}
\usepackage{makecell}
\usepackage{dsfont}
\usepackage{epstopdf}
\usepackage{tikz}
\usetikzlibrary{arrows.meta,positioning,calc}
\usepackage{algorithm}
\usepackage{algpseudocode}
\usepackage{mathrsfs}
\usepackage{wasysym}
\usepackage{threeparttable}
\usepackage{float}
\usepackage{tabularray}
\usepackage{longtable}
\usepackage{needspace}
\usepackage[disable,textsize=tiny]{todonotes}
\mathtoolsset{showonlyrefs=true}
\definecolor{SparkDiffusionRed}{HTML}{B9472F}
\definecolor{SparkDiffusionOrange}{HTML}{E97845}
\definecolor{SparkDiffusionLight}{HTML}{FFF0E9}
\definecolor{SparkDiffusionHeader}{HTML}{FFF7F2}
\definecolor{SparkDiffusionGray}{HTML}{666666}

\newcommand{\best}[1]{\textbf{#1}}
\makeatletter
\@ifundefined{c@theorem}{%
  
}{}

\makeatother

\definecolor{zzk}{RGB}{67,151,143}

\definecolor{lyx}{RGB}{221,160,221}

\definecolor{hyp}{RGB}{122,20,122}

\newtcolorbox{limbox}{
  enhanced,
  breakable,
  colback=orange!5,
  frame hidden,
  boxrule=0pt,
  borderline west={1.2pt}{0pt}{orange!55!black},
  sharp corners,
  left=6pt,
  right=2pt,
  top=3pt,
  bottom=3pt,
  before skip=5pt,
  after skip=5pt
}

\newtcolorbox{limbox_blue}{
  enhanced, breakable, colback=blue!5,
  frame hidden, boxrule=0pt,
  borderline west={1.2pt}{0pt}{blue!55!black},
  sharp corners, left=6pt, right=2pt, top=3pt, bottom=3pt,
  before skip=5pt, after skip=5pt
}

\providecommand{\method}{\textsc{SparkDiffusion}\xspace}

\title{SparkDiffusion: Mitigating the High-Sparsity Trap --- A Unified Framework for up to $265\times$ Single-GPU Acceleration of Visual Generation}
\author[13]{Yuxi Liu$^*$}
\author[23]{Haoyu Li$^*$}
\author[1]{Zekun Zhang$^*$}
\author[3]{Tengxu Sun$^\dagger$}
\author[1]{Yixiang Cai}
\author[35]{Jiayong Li}
\author[6]{Yifei Xia}
\author[4]{Tianle Liu}
\author[3]{Baole Ai}
\author[3]{Ang Wang}
\author[3]{Jiamang Wang}
\author[3]{Lin Qu}
\author[2]{Kai Zhang}
\author[1]{Kun Yuan$^\dagger$}
\author[6]{Bin Cui$^\dagger$}

\affiliation[1]{Peking University, Melon Group}
\affiliation[2]{Tsinghua University}
\affiliation[3]{Alibaba Group}
\affiliation[4]{University of Electronic Science and Technology of China}
\affiliation[5]{Harbin Institute of Technology}
\affiliation[6]{Peking University}

\appto{\affiliationlist}{%
  \par\vskip 1.5mm
  {\affiliationfont\sffamily\bfseries
   Project Page:
   \href{https://sparkdiffusion.github.io/}{\nolinkurl{https://sparkdiffusion.github.io/}}%
   \\
   GitHub:
   \href{https://github.com/AlibabaResearch/SparkDiffusion}{\nolinkurl{https://github.com/AlibabaResearch/SparkDiffusion}}%
   \\
   Weights:
   \href{https://huggingface.co/collections/alibabagroup/sparkdiffusion}{\nolinkurl{https://huggingface.co/collections/alibabagroup/sparkdiffusion}}%
  }%
}
\paperemail{%
\begin{tabular}[t]{@{}ll@{}}
\href{mailto:yuxiliu666@stu.pku.edu.cn}{yuxiliu666@stu.pku.edu.cn}
&
\qquad
\href{mailto:hy-l24@mails.tsinghua.edu.cn}{hy-l24@mails.tsinghua.edu.cn}
\\[-1pt]
\href{mailto:zhangzekun@std.uestc.edu.cn}{zhangzekun@std.uestc.edu.cn}
&
\qquad
\href{mailto:lingyan.sk@alibaba-inc.com}{lingyan.sk@alibaba-inc.com}
\end{tabular}%
}

\correspondence{Tengxu Sun, Kun Yuan, Bin Cui}
\firstpagenotes{$^*$Equal contribution. $^\dagger$Corresponding author.}
\abstract{
Video diffusion transformers are expensive because attention dominates long spatiotemporal token sequences.
We identify the \emph{high-sparsity trap}: at extreme attention sparsity, step-local training losses keep decreasing while terminal generation quality stagnates or degrades.
The trap is one of supervision: the dominant terminal errors originate in the high-noise structure-generation stage, and terminal-aligned training corrects terminal errors that substantially extended step-local training cannot.
This yields a simple staging principle: \emph{first adapt the sparse architecture into a coarse prior, then correct the terminal distribution}.
We instantiate the principle as \method, a unified acceleration framework for visual generation that combines a short sparse warm-up, few-step trajectory-mixed distillation, and FP8 quantization with fused kernels.
\method sustains $97\%$ attention sparsity with strong visual quality on long-sequence 720P generation across Wan2.1/Wan2.2 backbones and T2V/I2V tasks, and $90\%$ sparsity on Wan2.1-T2V-1.3B-480P. With 3-step CFG-free inference, \method achieves a $265\times$ end-to-end speedup over the 50-step CFG dense baseline for Wan2.1-T2V-14B-720P on a single RTX~5090 ($220\times$ on H100), and denoises a Wan2.1-T2V-1.3B-480P video in $1.3$s.
}

\begin{document}

\maketitle
\vspace{-0.3cm}
\begin{figure}[H]
    \centering
    \includegraphics[width=0.94\linewidth]{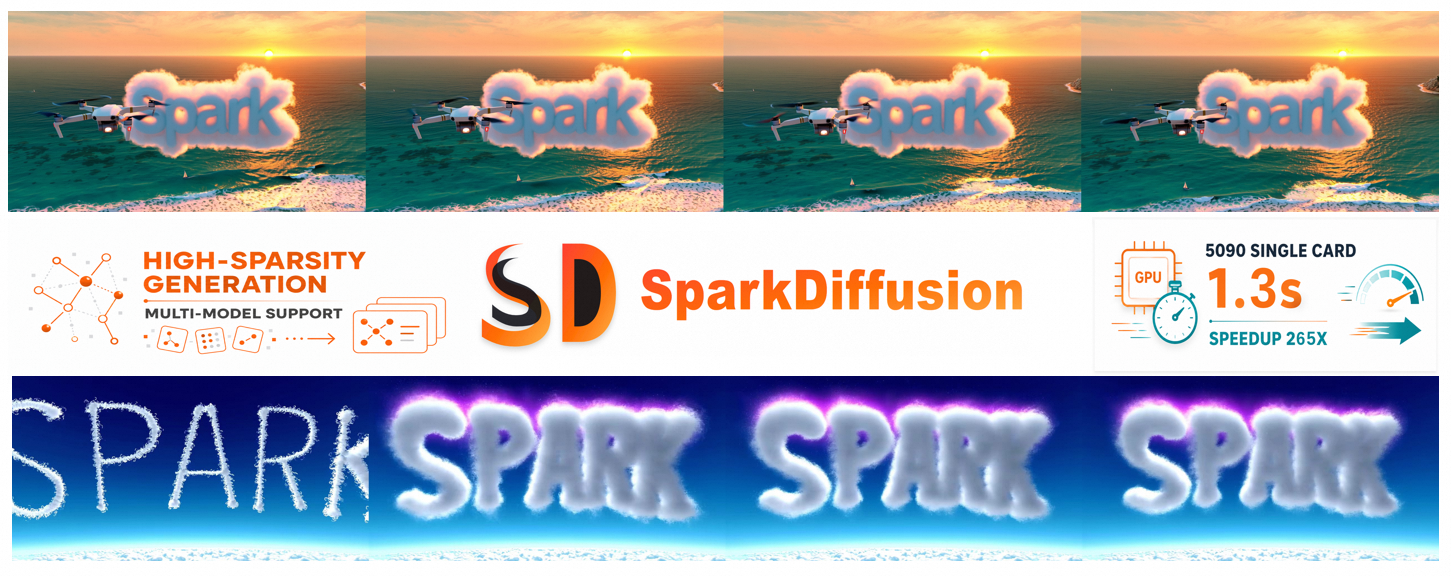}

\caption{
\textbf{SparkDiffusion teaser.} From one pretrained dense video DiT, \method delivers high quality frames across Wan2.1/Wan2.2, T2V/I2V, and 480P/720P with 3-step inference, sustaining $97\%$ attention sparsity on long-sequence 720P models and $90\%$ on the compact Wan2.1-T2V-1.3B-480P.
}

    \label{fig:Spark_demo}
\end{figure}

\begin{figure}[H]
    \centering
    \includegraphics[width=0.9\linewidth]{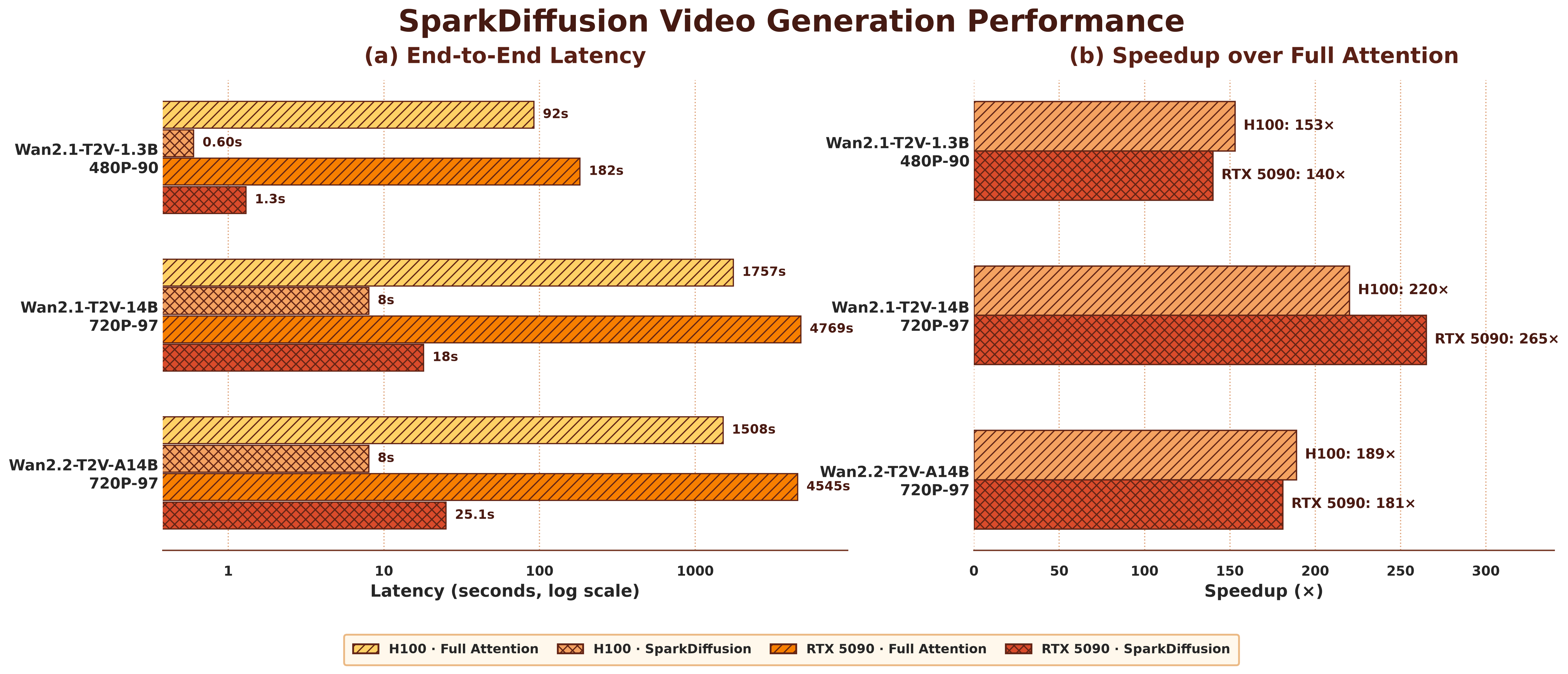}
\caption{
   End-to-end latency and speedup of \method over Full Attention  across Wan video generation models on NVIDIA H100 and RTX 5090 GPUs. Suffixes ``90'' and ``97'' denote attention sparsity levels. Speedup is the latency ratio of the full \method stack (few-step distillation, attention sparsity, and FP8) against the dense baseline; NFE conventions follow footnotes~\ref{fn:fn1}.
}

    \label{fig:sparkdiffusion_latency_speedup_h100_rtx5090}
\end{figure}

\section{Introduction}
\label{sec:intro}

\begin{wrapfigure}{r}{0.50\textwidth}
    \centering
    \vspace{-6pt}
    \includegraphics[width=\linewidth]{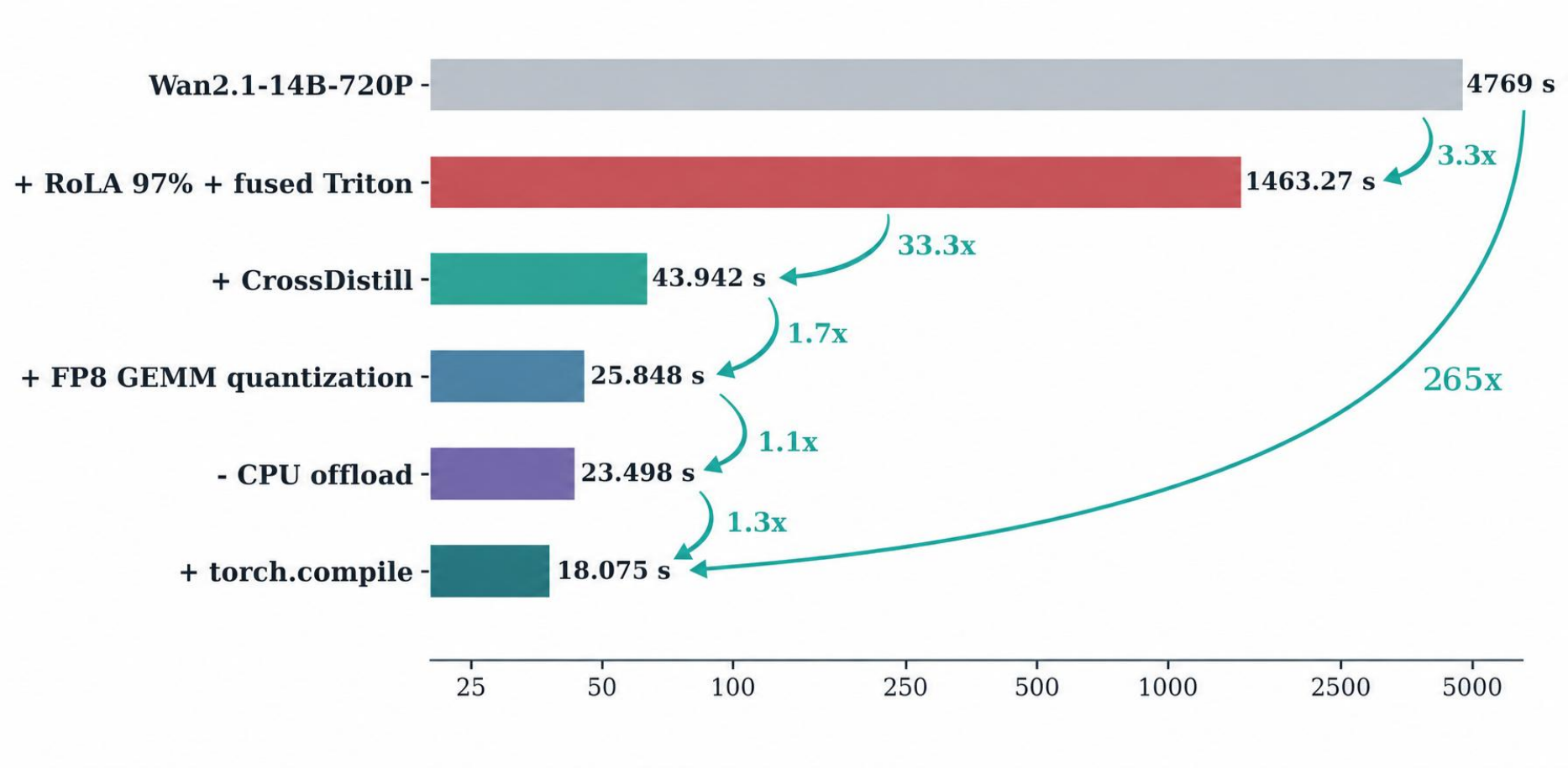}
\caption{
Headline speedup. On Wan2.1-T2V-14B-720P, \method composes three multiplicative factors---3-step CFG-free distillation, $97\%$ compensated sparse attention, and fused FP8 kernels---into a measured $265\times$ end-to-end speedup over the dense baseline (footnotes \ref{fn:fn1}).
}
    \label{fig:headline_speedup}
    \vspace{-8pt}
\end{wrapfigure}

Video diffusion transformers~\citep{peebles2023scalable} are a central architecture for high-quality text-to-video and image-to-video synthesis~\citep{wan2025wan,kong2024hunyuanvideo,yang2025cogvideox}, but their inference cost is dominated by attention over long spatiotemporal sequences at every sampling step.
Sparse attention reduces this cost by skipping redundant query--key interactions, yet conventional designs lose global context and degrade structure when pushed to very high sparsity~\citep{xi2025sparse,yang2026sparse,zhang2025spargeattention}.
Compensated sparse attention mitigates this problem by coupling a high-energy sparse branch with a lightweight low-rank or linear branch~\citep{zhang2026sla,zhang2026sla2,fang2026salad,liu2026ropeslr}, enabling around $90\%$ attention sparsity.

Since block-sparse attention computes only a fraction $1-s$ of query--key blocks, the sparse-branch computation scales roughly linearly with $1-s$.
Increasing sparsity from $90\%$ to $97\%$ reduces the retained sparse blocks from $10\%$ to $3\%$, cutting the block-wise sparse-attention computation to about $30\%$ of that at $90\%$ sparsity.
This reduction translates directly into end-to-end latency gains: on Wan2.1-T2V-14B-720P with 3-step inference, the $90\%\!\to\!97\%$ stretch alone reduces latency from $10.5$\,s to $8.0$\,s on H100 and from $23.7$\,s to $18.0$\,s on RTX~5090 (Fig.~\ref{fig:sparkdiffusion_latency_speedup_h100_rtx5090})---an additional $1.31$--$1.32\times$ speedup that is available only if quality survives extreme sparsity.\footnote{Unless otherwise stated, all Full Attention results in this paper use 50 sampling steps with classifier-free guidance, i.e., two model evaluations per sampling step (NFE${}=50\times2=100$); all accelerated results are CFG-free with 3-step inference (NFE${}=3$). The Full Attention baseline runs in BF16 with the standard high-performance dense attention kernel on each GPU---FlashAttention-2 on RTX~5090 and FlashAttention-3 on H100 \label{fn:fn1}}
However, extreme sparsity also exposes a second barrier: a sparse model may have fast inference and converged step-local training loss, while its terminal videos still exhibit broken structure, semantic drift, mosaic textures, and temporal flicker.
We refer to this regime as the \emph{high-sparsity trap}: a sparse generator can converge under step-local supervision while still failing terminally.

This raises the central question:

\begin{limbox_blue}
\textbf{Can visual generators maintain quality at $90\%+$ attention sparsity, and can such sparsity be converted into measured wall-clock speedup?}
\end{limbox_blue}

\begin{figure}[t]
    \centering
    \includegraphics[width=0.95\linewidth]{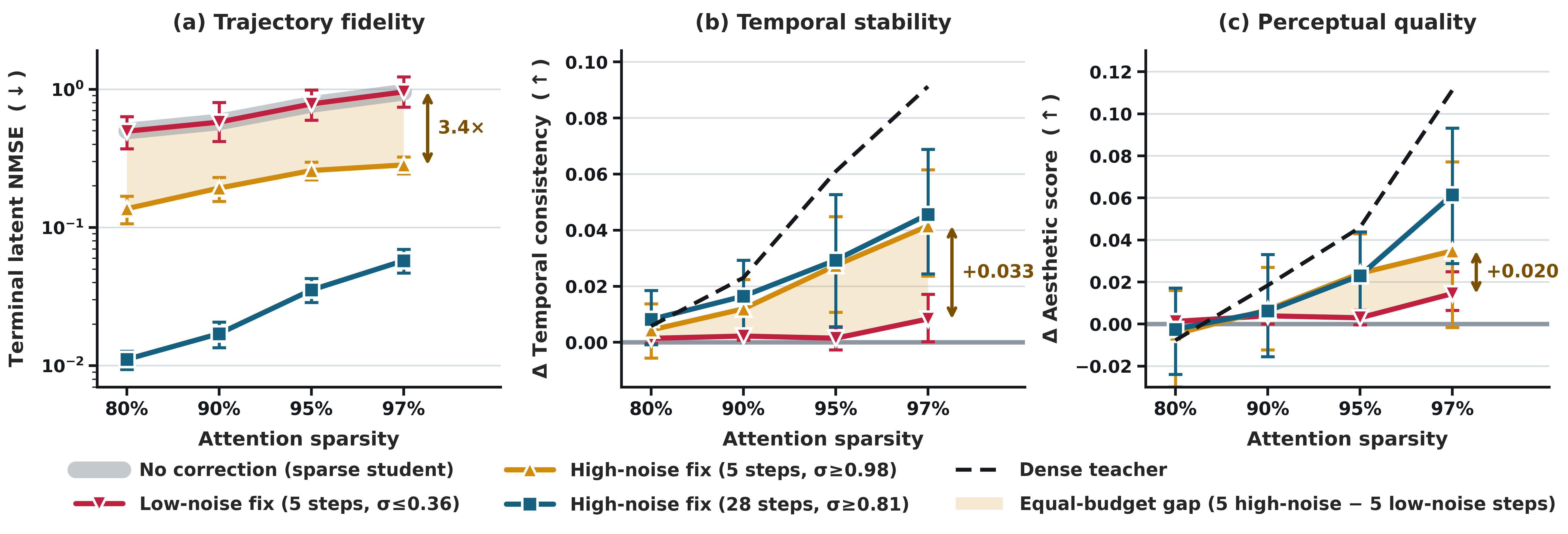}
\caption{
    \textbf{The high-sparsity trap is rooted in high-noise structure errors.}
\emph{Oracle probe}: from the same initial noise, the sparse model's
velocity is replaced by the dense teacher's velocity inside a short window
of sampling steps (``\emph{fix}''; $\sigma$ in the legend is the
flow-matching noise level; Wan2.1-T2V-14B-480P, 50-step sampler,
32 prompts $\times$ 4 seeds; error bars: 95\% CI).
(a)~Terminal error to the dense reference output (latent MSE, $\downarrow$;
the teacher is the zero reference and is off the log scale).
(b,c)~Gains over the uncorrected sparse model ($\uparrow$; dashed: dense
teacher).
At $97\%$ sparsity, fixing the five highest-noise steps removes most of the
terminal error, while fixing the five lowest-noise steps leaves it
essentially unchanged (beige band: equal-budget gap between the two 5-step
fixes); a wider 28-step high-noise window recovers nearly all of it. Both the student and the teacher are sampled with classifier-free guidance, and the intervention replaces the student's guided velocity with the teacher's guided velocity, isolating sparsity-induced error from any CFG-removal effect.}
    \label{fig:headline_high_vs_low_noise}
\end{figure}

Figure~\ref{fig:headline_high_vs_low_noise} provides a mechanistic diagnosis: at extreme sparsity, replacing the sparse student's velocity with the dense teacher's velocity only in a high-noise window substantially reduces the terminal error, whereas the same correction in a low-noise window yields limited gain (setup and details in Sec.~\ref{sec:trap}).

Representation is not the limiting factor at this operating point: video attention exhibits a strong sparse-plus-low-rank structure in which a high-energy sparse branch captures salient semantic interactions and a low-rank branch models the residual background context, so compensated sparse attention can in principle sustain sparsity well beyond $90\%$ on long spatiotemporal sequences~\citep{liu2026ropeslr}.
Once architectural compensation makes such extreme sparsity reachable in representation, the remaining bottleneck is supervision.
Sparse video diffusion models typically inherit flow matching~\citep{lipman2022flow}, which is \emph{step-local}: it supervises velocities at individual noise levels but does not directly constrain the terminal sample.
For dense models, residual velocity errors are often small enough that indirect supervision is benign.
For highly sparse models, however, errors can be larger and more systematic, especially in the high-noise regime, so a model can minimize step-local loss while still producing terminal samples misaligned with the data distribution.

This observation motivates \method, a staged post-training framework.
A short sparse warm-up adapts the dense backbone to the sparse architecture and provides a trainable coarse prior.
Trajectory-mixed distillation then corrects the terminal distribution by applying high-noise structural alignment and low-noise distribution matching.
Finally, fused FP8 deployment converts the reduced computation into measured wall-clock speedup.
The same staged recipe applies to T2V/I2V tasks and to both dense Wan2.1 and MoE-based Wan2.2 backbones (Sec.~\ref{sec:exp}).

\textbf{Contributions.}
\begin{itemize}
    \item \textbf{Diagnosis: the high-sparsity trap.}
    We identify and diagnose the \emph{high-sparsity trap}: at extreme attention sparsity, step-local training keeps reducing the validation loss while terminal generation quality stagnates or degrades.
     Terminal-aligned post-training mitigates the trap where substantially extended step-local training cannot.
    \item \textbf{Framework: \method.}
    We build \method, a unified acceleration framework for visual generation that chains three complementary stages into a single pipeline: a short sparse warm-up that adapts a pretrained dense backbone into a coarse sparse prior, few-step trajectory-mixed distillation that corrects the terminal distribution, and FP8 quantization with fused kernels that converts the saved computation into measured wall-clock speedup.
    \item \textbf{Results.}
 \method sustains $97\%$ attention sparsity on long-sequence 720P generation with strong visual quality: on Wan2.1-T2V-14B-720P it delivers a $265\times$ end-to-end speedup on a single RTX~5090 and $220\times$ on H100, and it generates Wan2.1-T2V-1.3B-480P video end-to-end in $1.3$ seconds at $90\%$ sparsity.
    The framework covers Wan2.1/Wan2.2 backbones, T2V/I2V tasks, 480P/720P resolutions, and both RTX~5090 and H100 GPUs.
\end{itemize}

\textbf{Extension to auto-regressive video diffusion.}
The framework is also compatible with auto-regressive (AR) video diffusion~\cite{yin2025slow}.
Compensated sparse attention can be restricted to a causal window.
Trajectory-mixed distillation acts on the noise axis rather than the temporal axis, so it can be combined with self-forced AR training~\cite{huang2026self}.
Since high-noise structural errors can propagate across chunks, terminal alignment is arguably even more important in AR generation.

\section{Preliminaries}
\label{sec:preliminaries}

The following ingredients are central to our discussion: flow-matching training, block-sparse attention, and compensated sparse attention.

\textbf{Flow matching.} Following flow matching~\citep{lipman2022flow}, we consider a video diffusion transformer that predicts a velocity field $F_\theta(x,t)$ on latent tokens $x \in \mathbb{R}^{L \times d}$, where $L$ is the spatiotemporal token length.
Let $F_{\theta_0}$ denote a pretrained dense checkpoint.
Given a clean sample $x_0 \sim p_{\mathrm{data}}$ and Gaussian noise $\epsilon \sim \mathcal{N}(0,I)$, flow matching forms the interpolated noisy sample
\begin{equation}
x_t = (1-t)x_0 + t\epsilon,
\qquad t \in [0,1],
\end{equation}
and trains the model to predict the path velocity $\epsilon - x_0$:
\begin{equation}
\label{eq:fm-loss}
\mathcal{L}_{\mathrm{FM}}(\theta)
=
\mathbb{E}_{x_0,\epsilon,t}
\left[
w(t)
\left\|
F_\theta(x_t,t) - (\epsilon - x_0)
\right\|_2^2
\right],
\end{equation}
where $w(t)$ is an optional timestep weighting function.
This objective supervises velocities at sampled noise levels; the terminal sample $\hat{x}_0$ is obtained only after composing many sampling steps.

\textbf{Block sparsity.}
Throughout this paper, \emph{attention sparsity} refers to \emph{block sparsity}: the attention matrix is partitioned into contiguous blocks along query and key dimensions, and a fraction $s$ of these blocks are skipped during computation.
We report sparsity as the fraction of blocks not evaluated by the sparse branch.

\textbf{Compensated sparse attention.}
A major cost in a video DiT is dense self-attention,
\[
A(Q,K,V)
=
\operatorname{softmax}\!\left(\frac{QK^\top}{\sqrt{d}}\right)V,
\]
which costs $O(L^2)$.
Compensated sparse attention replaces it by
\begin{equation}
\label{eq:comp-sparse}
\widetilde{A}(Q,K,V)
=
A_M(Q,K,V)
+
g_\gamma \odot C_\psi(Q,K,V).
\end{equation}
Here $A_M$ is a sparse branch that computes attention on a fraction $1-s$ of query-key blocks selected by $M_\phi$, e.g., via top-$k$ selection, blockwise scores, or fixed spatiotemporal patterns.
$C_\psi$ is a lightweight compensation branch, such as linear/low-rank attention or pooled summaries, designed to recover global context discarded by the sparse branch.
Gate $g_\gamma$ fuses the two branches and is typically initialized near zero.

\textbf{Training sparse attention.}
Starting from the dense checkpoint $F_{\theta_0}$, we insert the compensated sparse attention modules into the backbone and train the resulting model with the same task loss as in \eqref{eq:fm-loss}.
Some variants replace the target $\epsilon - x_0$ with the dense teacher prediction $F_{\theta_0}(x_t,t)$, but the objective remains a step-local velocity-regression loss.

\section{The high-sparsity trap}
\label{sec:trap}

We use the term \emph{high-sparsity trap} to describe a specific failure regime at extreme sparsity.
Mitigating the trap unlocks the largest remaining latency gain in sparse attention (Sec.~\ref{sec:intro}).

\begin{limbox}
\textbf{The high-sparsity trap.}
We say that a sparse video DiT falls into the high-sparsity trap when, under a compensated sparse architecture, increasing attention sparsity beyond a threshold causes terminal generation quality to stagnate or degrade while the step-local validation loss remains low or continues to decrease, and the terminal error persists even under substantially extended step-local training budgets.
\end{limbox}
\textbf{Generality across sparse designs.}
RoLA~\citep{zhang2026rola} shares its sparse branch with representative trainable sparse-attention methods such as VSA~\citep{zhang2026faster} and SLA~\citep{zhang2026sla}: all three compute exact softmax attention on a retained subset of query--key blocks and differ only in the compensation mechanism.
The high-sparsity trap observed on RoLA therefore reflects this whole family: Appendix~\ref{app:generality} reproduces the same failure pattern on VSA- and SLA-style designs, and the staged recipe of Sec.~\ref{sec:method} restores terminal quality on both.

Operationally, we identify this regime by three signatures:
(i) the sparse model's step-local validation loss has plateaued or continues to decrease;
(ii) terminal generation quality, measured by automatic metrics or paired terminal error, is significantly worse than the dense or lower-sparsity baseline;
and (iii) extending step-local training does not substantially recover the terminal quality.

\textbf{Terminal error.}
Throughout this section, the \emph{terminal error} denotes the paired MSE between the sparse and dense models' terminal outputs under matched initial noise; unless otherwise stated, it is measured in latent space.
Both models use identical settings---50 steps with classifier-free guidance (CFG scale $5.0$)---so the terminal error reflects sparsity-induced error alone, not CFG removal or step reduction.

\subsection{Empirical signature: terminal drift at extreme sparsity}
\begin{wraptable}{r}{0.47\linewidth}
    \vspace{-14pt}
    \centering
    \scriptsize

    \caption{
    Terminal error (paired latent MSE) at $97\%$ sparsity with rank-64 vs.\ full-rank ($r{=}128$) compensation
    (Wan2.1-T2V-14B-480P, 10{,}000 steps).
    }
    \label{tab:fullrank}

    \begin{tblr}{
        width=\linewidth,
        colspec={Q[l,wd=2.80cm] X[c] X[c]},
        colsep=1.8pt,
        rows={rowsep=2.2pt},
        row{1}={bg=SparkDiffusionHeader,font=\bfseries},
        hline{1}={1.25pt,SparkDiffusionRed},
        hline{2}={0.7pt,SparkDiffusionRed},
        hline{4}={1.25pt,SparkDiffusionRed}
    }

    Compensation
    & Term.\ error $\downarrow$
    & VBench $\uparrow$
    \\

    Rank-64 (default)
    & 0.124
    & 80.92
    \\ 

    Full rank ($r{=}128$)
    & 0.121
    & 81.05
    \\ 

    \end{tblr}

    \vspace{-10pt}
\end{wraptable}
High sparsity is not ruled out by architecture alone: video attention exhibits a strong
sparse-plus-low-rank structure that supports 90\%+ sparsity~\cite{liu2026ropeslr}, and indeed
the 80\% and 90\% models remain close to the dense baseline under the same training
(Figure~\ref{fig:wan21_sparsity_crash_main_step2000}); the failure appears only when the same architecture
family is pushed to 95\% and 97\%. Two observations identify the training signal,
rather than the representation, as the bottleneck. 
First, widening the compensation
branch from rank-64 to full rank leaves the terminal error essentially unchanged
(Table~\ref{tab:fullrank}): the model can already express far richer
corrections than the ones it converges to. Second, the objective it is given is still
being optimized: the step-local validation loss continues to decrease while terminal
quality stalls (Figure~\ref{fig:wan21_reference_style_matrix}).

\textbf{Evidence in a controlled setting.}
On toy 2D sequence manifolds, a $95\%$ sparse model trained until validation loss plateaus still exhibits structural deviations introduced in the high-noise stage; these deviations become more visible as denoising proceeds (Figure~\ref{fig:wrap}).
Figure~\ref{fig:mainfld_showcase} shows the qualitative counterpart.

\begin{wrapfigure}{R}{0.52\textwidth}
    \centering
    \includegraphics[width=0.45\textwidth]{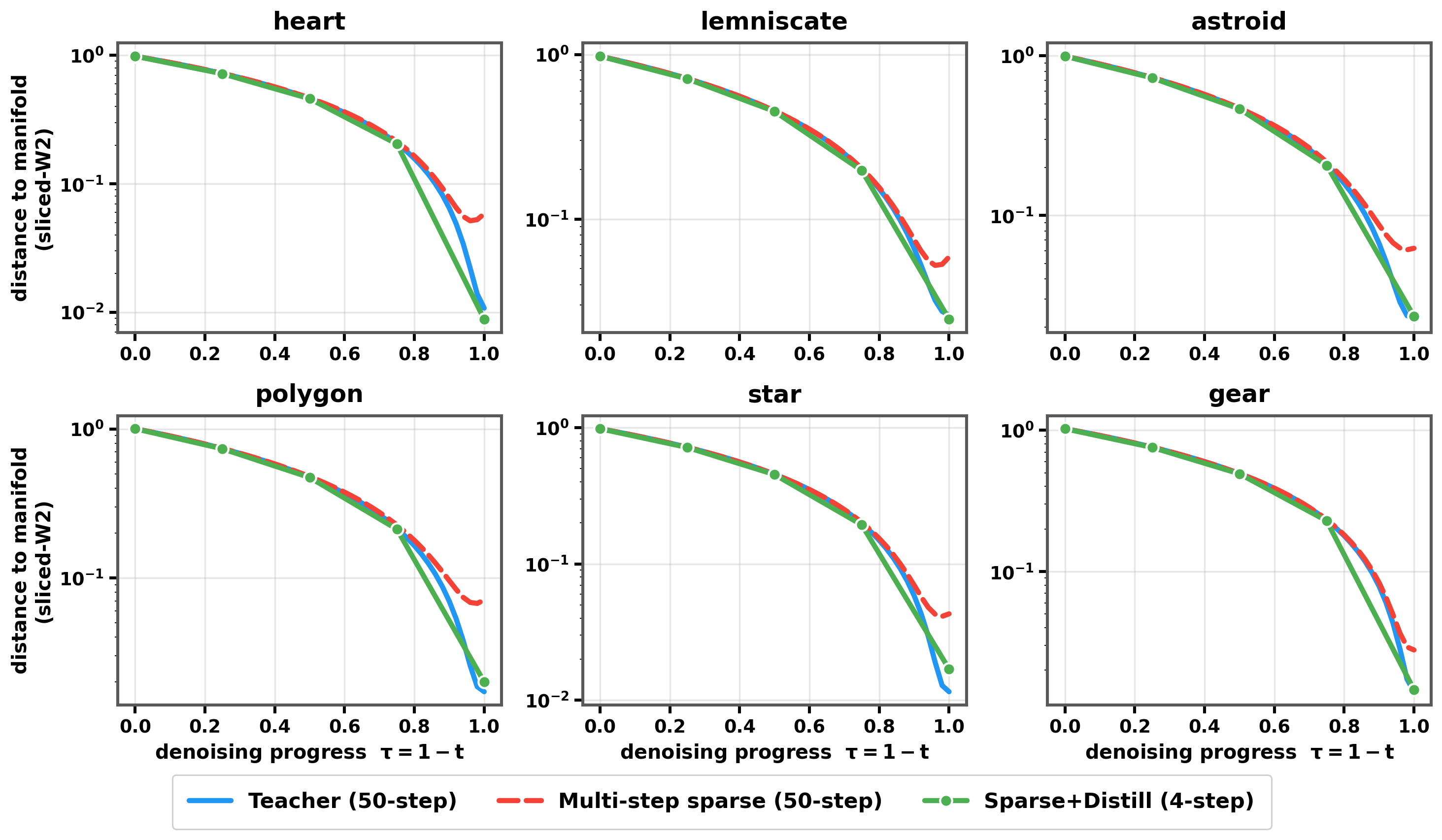}
    \caption{
    Sliced-$W_2$ distance to the data distribution along the denoising trajectory on six toy 2D sequence manifolds.
    The sparse model is trained until validation loss plateaus, while the distilled student is obtained by trajectory-mixed distillation.
    }
    \label{fig:wrap}
\end{wrapfigure}

\textbf{Real-video evidence under controlled training budgets.}
On a real video backbone, all sparse models are trained with the same step-local task
loss~\eqref{eq:fm-loss} from the same dense checkpoint on an enlarged training set of
$\approx$30{,}000 videos ($15\times$ the production warm-up set), 
and training is deliberately
extended to 10{,}000 steps---$40\times$ the step-local budget of the production warm-up
($\approx$250 steps), and more than the total step count of the full two-stage pipeline
of Sec.\ref{sec:method} (8{,}250 steps; Tables~\ref{tab:stage1_training_config}
and~\ref{tab:stage2_training_config}).


\begin{figure}[t]
    \centering
    \includegraphics[width=0.95\linewidth]{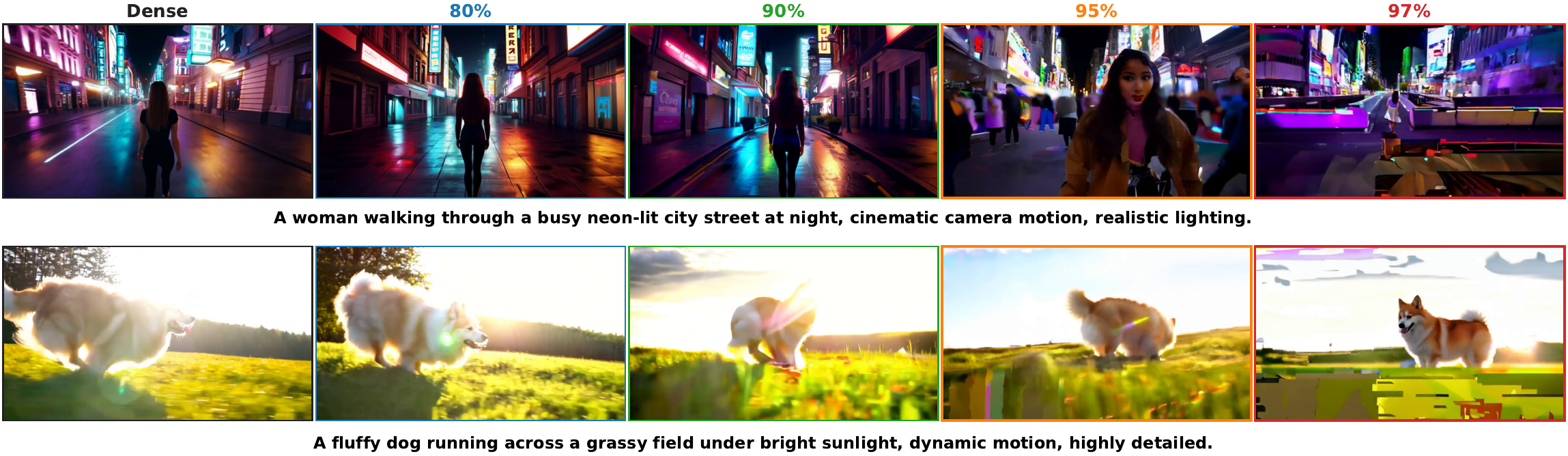}
    \caption{
           Video frame comparison on Wan2.1-T2V-14B-480P under attention sparsity levels
       $80\%$, $90\%$, $95\%$, and $97\%$.
    }
    \label{fig:wan21_sparsity_crash_main_step2000}
\end{figure}

\textbf{Where the error is injected.}
Figure~\ref{fig:headline_high_vs_low_noise} localizes the terminal error with an oracle intervention: for a fixed initial noise, the sparse student's velocity is replaced by the dense teacher's velocity on a contiguous noise interval, while the sparse student is kept elsewhere.
At $97\%$ sparsity, correcting the five highest-noise steps removes most of the terminal error, whereas the equal-budget low-noise correction leaves it essentially unchanged; a wider high-noise window recovers nearly all of it.
Structural errors are therefore injected during high-noise structure generation and then amplified by later steps.

\begin{figure}[t]
    \centering
    \includegraphics[width=0.85\linewidth]{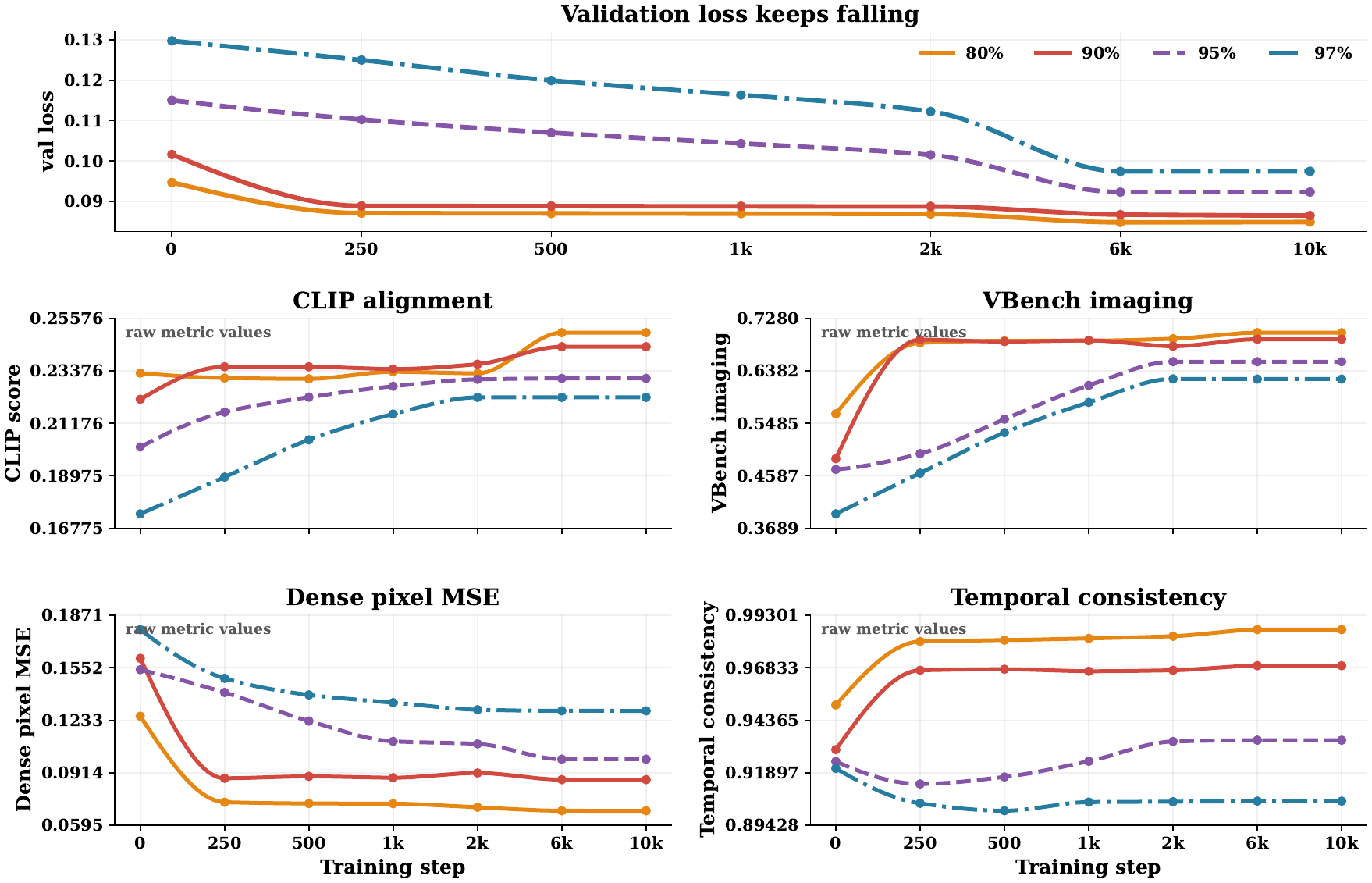}
    \caption{
   Validation loss and raw quality metrics for Wan2.1-14B across sparsity levels.
   The ``Dense pixel MSE'' panel reports the terminal error of Sec.\ref{sec:trap} after VAE decoding, sharing the pixel-space footing of the other quality panels.
    }
    \label{fig:wan21_reference_style_matrix}
\end{figure}
Together, these diagnostics show that step-local validation loss is a poor proxy for terminal quality at extreme sparsity.

\subsection{Why step-local supervision may be insufficient}

Intuitively, the step-local loss measures per-step error magnitudes, whereas the terminal sample depends on how those errors are transported and summed along the sampling trajectory.
Appendix~\ref{sec:insight_terminal_misalignment} formalizes this mechanism in a deliberately simplified surrogate.
After a step-local optimum, a terminal-aligned signal can reduce the surviving terminal error at first order through the same trainable parameters, but it cancels the terminal-visible projection of the error rather than removing the underlying velocity error.

\subsection{From diagnosis to a staged remedy}
We call a training signal \emph{terminal-aligned} if its supervision target at noise
level $t$ is constructed from the sampling trajectory \emph{beyond} $t$---the model's
prediction at a later trajectory state, a teacher-composed trajectory segment, or
the terminal output itself---so that the loss constrains how the prediction at $t$
\emph{composes} with the remainder of the trajectory toward the terminal state.
In contrast, step-local losses such as flow matching supervise the velocity at
$(x_t, t)$ against a closed-form, pointwise target that is independent of the sampling
trajectory, and are therefore agnostic to how per-step errors compose and accumulate
(Appendix~\ref{sec:insight}).
Under this definition, consistency objectives are terminal-aligned toward the terminal
state of their noise interval, and distribution-matching objectives are terminal-aligned
toward the global terminal output $\hat{x}_0$.
Sparse training with step-local loss provides a coarse prior but does not anchor
the terminal distribution at extreme sparsity.
Applying few-step distillation post-training to the sparse model with this prior is therefore a natural way to mitigate the high-sparsity trap.
The remedy is staged: a short sparse warm-up adapts the high-sparsity architecture into
a usable prior, and trajectory-mixed distillation (Sec.~\ref{sec:method}) then aligns
the teacher's velocity field in the high-noise segment and matches the data distribution
in the low-noise segment.

Figure~\ref{fig:wan21_distill_comparison} provides qualitative evidence: at $97\%$ attention sparsity on Wan2.1-T2V-14B-720P, the few-step model after trajectory-mixed distillation is visually much closer to the dense teacher than the multi-step sparse model.

\begin{figure}[h]
    \centering
    \includegraphics[width=0.85\linewidth]{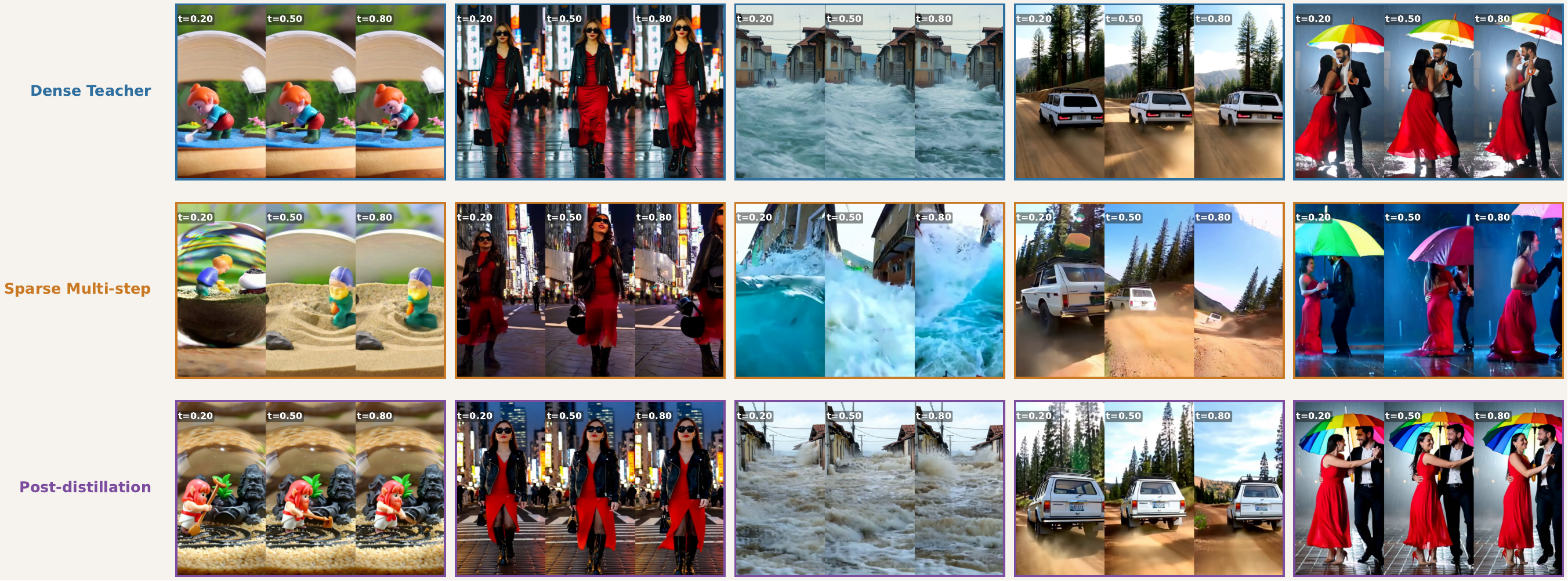}
    \caption{
        Qualitative comparison at $97\%$ attention sparsity on Wan2.1-T2V-14B-720P.
        We compare the dense teacher, the multi-step sparse model, and the few-step sparse model obtained after trajectory-mixed distillation.
    }
    \label{fig:wan21_distill_comparison}
\end{figure}

\section{Method}
\label{sec:method}

\method consists of two post-training stages and one deployment-time step.
Figure~\ref{fig:framework} gives an overview of the framework.
Together, they convert a pretrained dense video DiT into a fast, high-sparsity generator without retraining it from scratch.
Stage~1 performs a sparse-attention warm-up to obtain a coarse prior (Sec~\ref{subsec:sparse});
Stage~2 applies trajectory-mixed distillation to reduce sampling steps and correct the terminal distribution (Sec~\ref{subsec:distill});
Stage~3 uses FP8 quantization to convert reduced computation into wall-clock speedup (Sec~\ref{subsec:fp8}).
Each stage consumes the checkpoint of the previous one; the full schedule is in Sec~\ref{subsec:strategy}.

The framework is agnostic to the specific sparse selector, compensation branch, and distillation objective.
Our default instantiation uses RoLA~\citep{zhang2026rola} as the sparse module and the CrossDistill schedule~\citep{liu2026crossdistill} for distillation.

\begin{figure}[t]
    \centering
    \includegraphics[width=0.8\linewidth]{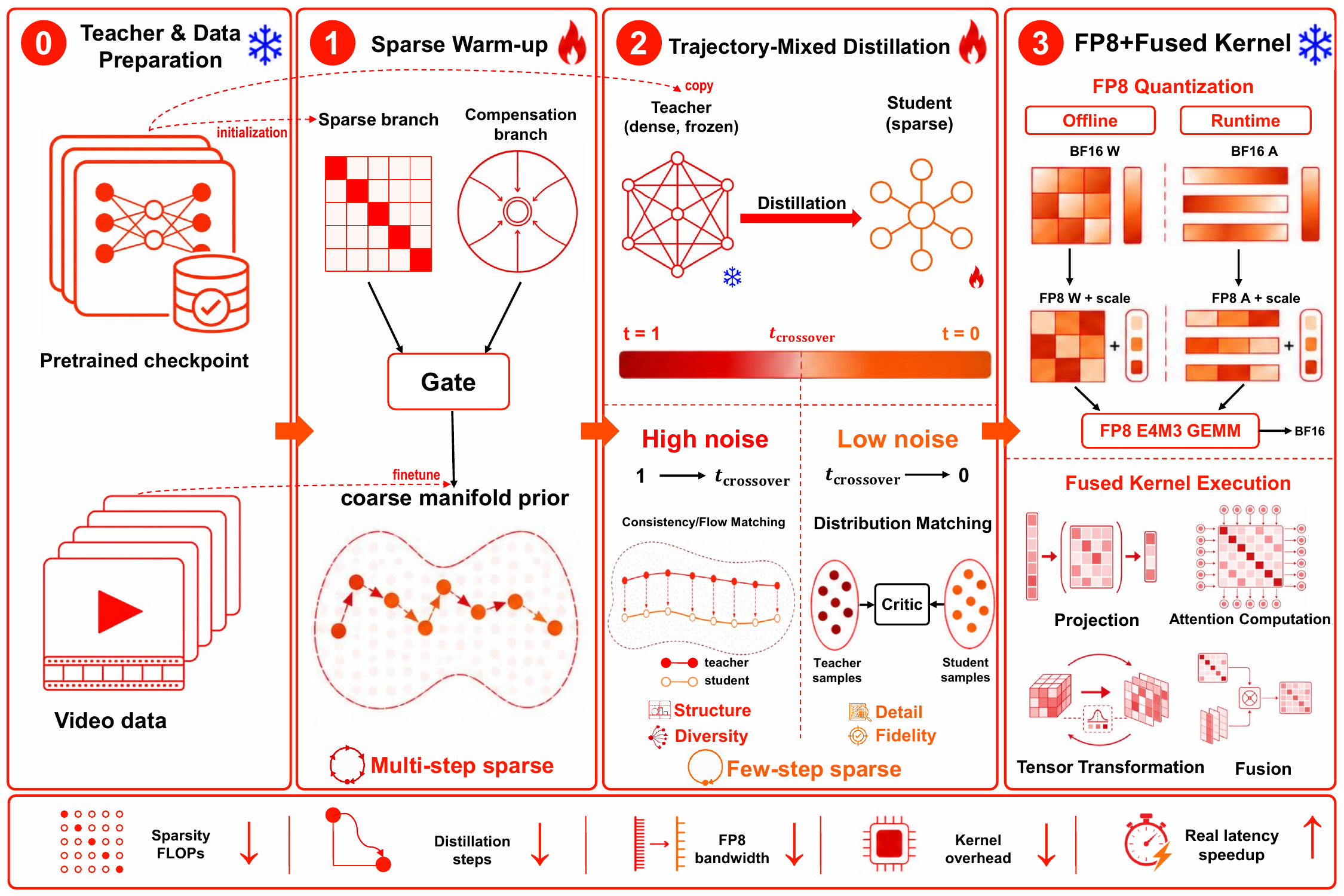}
    \caption{
   Overview of the \method framework (stages detailed in Sec.~\ref{subsec:sparse}--\ref{subsec:fp8}).
    }
    \label{fig:framework}
\end{figure}

\subsection{Compensated sparse attention}
\label{subsec:sparse}

Our sparse attention module follows RoLA~\citep{zhang2026rola}.
A block-sparse branch retains high-energy query--key blocks, while a low-rank branch recovers the global context discarded by sparsification.
The sparse branch computes
\begin{equation}
O_s = A_M(Q,K,V),
\end{equation}
where $A_M$ keeps only a small set of selected query--key blocks.

The compensation branch uses the backbone's pre-attention normalized queries and keys.
It applies low-rank projections, a pointwise nonlinearity, and then a rank-truncated 3D RoPE rotation:
\begin{equation}
\tilde q_i = R_r(p_i)\,\phi(P_q q_i^{\mathrm{norm}}),
\qquad
\tilde k_j = R_r(p_j)\,\phi(P_k k_j^{\mathrm{norm}}),
\qquad
C = \sum_j \tilde k_j v_j^\top,
\qquad
O_{lr,i} = \operatorname{RMSNorm}(\tilde q_i^\top C),
\end{equation}
where $\phi$ is SiLU, and $R_r(\cdot)$ reuses the first $r$ rotary coordinates of the pretrained 3D RoPE schedule.
This branch is linear in sequence length and preserves relative spatiotemporal position.

A token-wise gate $g_i$ initialized near zero fuses the two branches:
\begin{equation}
O_i = O_{s,i} + g_i\, r_i\, O_{lr,i},
\end{equation}
where $r_i$ is the RMS of the sparse output $O_{s,i}$.
Thus the module starts close to the sparse branch and opens the compensation path only where needed.

\subsection{Trajectory-mixed distillation}
\label{subsec:distill}

Stage-2 distillation must meet two distinct requirements.
First, escaping the high-sparsity trap requires a terminal-aligned signal (Sec.~\ref{sec:trap}); consistency-style and distribution-matching objectives are both terminal-aligned, so either suffices in principle (Sec.~\ref{subsec:exp_objective_ablation}).
Second, the choice among them is governed by the fidelity--diversity trade-off: pure distribution matching is mode-seeking and loses seed-level diversity~\citep{chen2026data}, while pure consistency matching preserves the teacher's trajectory but never aligns the terminal distribution.

We therefore adopt the trajectory-mixed schedule of CrossDistill~\citep{liu2026crossdistill}, which splits the sampling trajectory at a noise crosspoint: the high-noise segment is trained with PCM-style\cite{wang2024phased} consistency to preserve structure, motion, and diversity, while the low-noise segment is trained with DMD-style\cite{yin2024one} distribution matching to sharpen detail and correct terminal-visible errors.
We use a 3-step student---one high-noise PCM step and two low-noise DMD steps---and follow the default crosspoint, loss weighting, and optimization schedule of CrossDistill.
This assignment mirrors the diagnosis of Sec.~\ref{sec:trap}: terminal-visible errors originate in high-noise structure generation (Fig.~\ref{fig:headline_high_vs_low_noise}), so the high-noise step stays anchored to the teacher's coarse trajectory while the low-noise steps supply the terminal-aligned correction.

\subsection{FP8 quantization}
\label{subsec:fp8}

The deployment step converts the sparse model's reduced FLOPs into wall-clock speedup.
We quantize the linear projections in attention and the feed-forward layers to W8A8 FP8 (E4M3)~\citep{shen2024efficient}: weights carry static per-channel scales calibrated once offline, and activations are scaled dynamically per token at runtime.
Normalization layers, gating operations, and sparse mask selection remain in BF16.
The per-token activation quantization path is fused into a single custom kernel so that scaling and casting happen in one pass over memory.
This transform is applied purely at deployment time, after all training is complete, and with few-step sampling the quantization error has little room to accumulate.
All \method quality numbers reported in this paper (Tables~\ref{tab:main_vbench} and~\ref{tab:diversity}) are measured on the deployed model with FP8 quantization and fused kernels enabled; the BF16-to-FP8 quality delta is ablated in Appendix~\ref{app:fp8_ablation} and stays within $0.1$ VBench points.

\subsection{Training and deployment}
\label{subsec:strategy}
\begin{wrapfigure}{r}{0.50\textwidth}
    \centering
    \vspace{-12pt}
    \includegraphics[width=\linewidth]{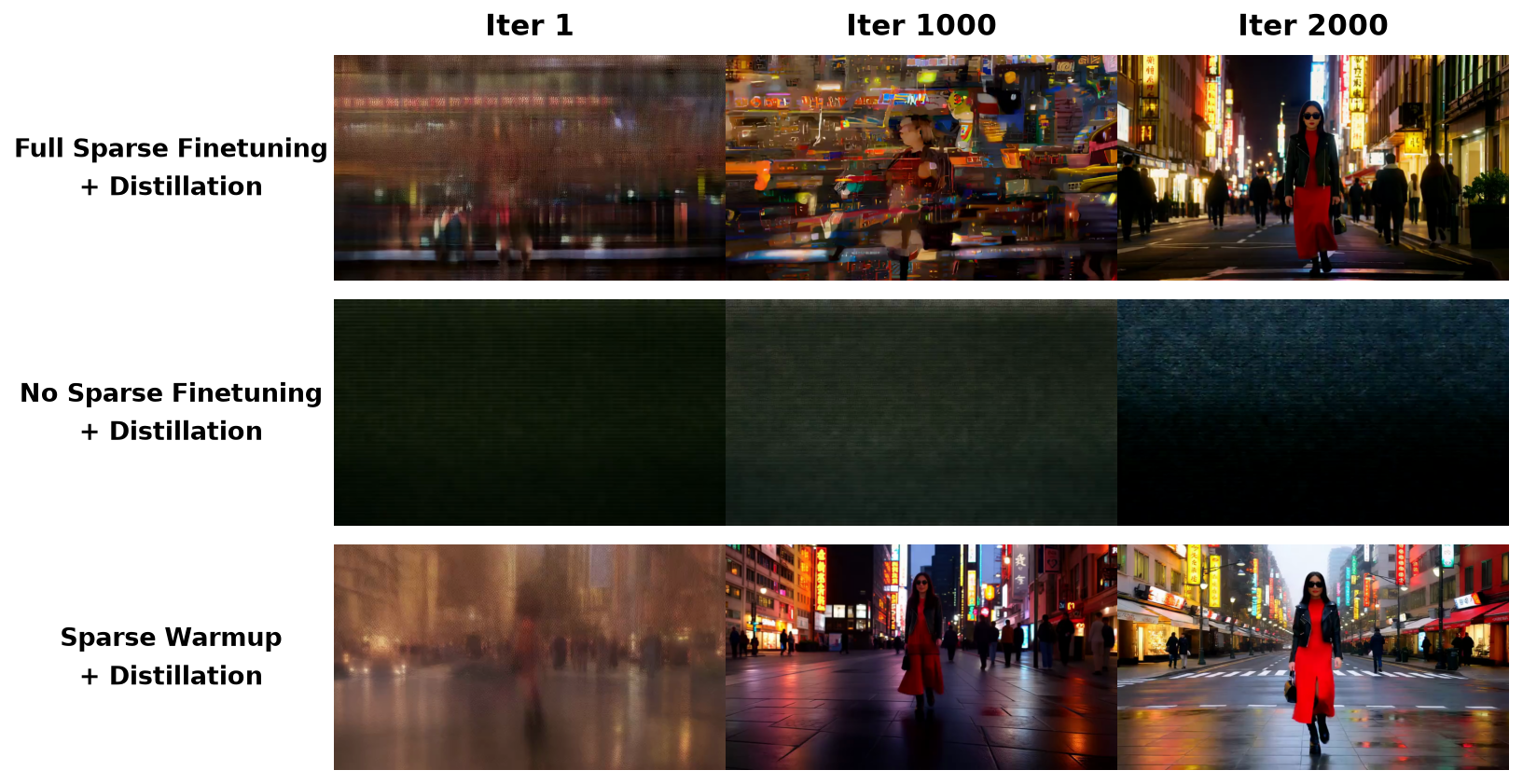}
    \vspace{-18pt}
    \caption{
        \textbf{Sparse warm-up provides the coarse prior for terminal-aligned distillation.}
        On Wan2.1-T2V-14B-480P, all models are followed by the same trajectory-mixed distillation stage; a short sparse warm-up establishes the usable coarse prior on which distillation builds.
    }
    \label{fig:warm_up_ablation}
    \vspace{-10pt}
\end{wrapfigure}

\textbf{Stage 1 (sparse warm-up).}
We insert the compensated sparse attention of Sec~\ref{subsec:sparse} into the dense backbone and fine-tune briefly with the step-local objective of Sec.~\ref{sec:preliminaries} (Eq.~\eqref{eq:fm-loss}, including its teacher-velocity variant) alone.
This stage adapts the backbone to the sparse architecture and establishes a coarse generative prior for the terminal-aligned stage that follows.
Figure~\ref{fig:warm_up_ablation} shows that a brief warm-up is sufficient to establish a usable prior for distillation, so we keep this stage short.

\textbf{Stage 2 (trajectory-mixed distillation).}
Starting from the warm-up checkpoint, we apply the trajectory-mixed distillation of Sec.~\ref{subsec:distill}, with the dense multi-step model frozen as the teacher.

\textbf{Deployment.}
The few-step model is quantized with the FP8 scheme of Sec~\ref{subsec:fp8} and used directly for inference.
For mixture-of-experts backbones such as Wan2.2, where separate expert groups are specialized for different noise ranges, each expert group is warmed up and distilled on its own noise range before quantization.

\section{Experiments}
\label{sec:exp}

\subsection{Main quantitative comparison}
\label{subsec:exp_main_table}

Table~\ref{tab:main_vbench} compares Full Attention, TurboDiffusion, FastWan (VSA), and \method under the same hardware and generation settings.\footnote{TurboDiffusion and FastWan provide official weights only for Wan2.1 backbones; no official release supports Wan2.2-T2V-A14B, so these baselines are omitted for that model. FastWan is the few-step sparse-attention video model implemented and released in the official VSA repository~\citep{zhang2026faster}: VSA denotes the training method and FastWan the released few-step sparse model. We therefore treat them as the same baseline and denote it as FastWan (VSA) throughout this paper.}
All baselines are evaluated at their official released operating point ($90\%$ sparsity); we report \method both at matched $90\%$ and at $97\%$.
At matched $90\%$ sparsity, \method improves over all baselines on every reported metric on Wan2.1-T2V-14B-720P.
On the compact short-sequence Wan2.1-T2V-1.3B-480P, \method at the same $90\%$ sparsity improves generation quality and inference speed simultaneously over both baselines: it attains the best VBench-2.0 at the lowest latency ($1.3$\,s on RTX~5090 and $0.6$\,s on H100, $1.5$--$2.2\times$ faster than the baselines).

On the large long-sequence models, we push sparsity to $97\%$.
All baselines are evaluated at their strongest official configuration ($90\%$ sparsity); since quality degrades monotonically with sparsity (Fig.~\ref{fig:wan21_sparsity_crash_main_step2000}), training the baselines at $97\%$ could only widen the quality gap in our favor.
At $97\%$, \method attains slightly better aggregate quality than the strongest $90\%$ baselines (on Wan2.1-T2V-14B-720P, VBench and VBench-2.0 relative to TurboDiffusion) at clearly lower latency, while remaining closest to the dense model in diversity (Table~\ref{tab:diversity}).
Measured end-to-end diffusion generation latency, excluding the text encoding
and VAE decoding stages and speedup are shown in Figures~\ref{fig:sparkdiffusion_latency_speedup_h100_rtx5090} and~\ref{fig:headline_speedup}.
These speedups are combined system results in which few-step distillation, attention sparsity, and FP8 contribute multiplicatively (Fig.~\ref{fig:headline_speedup}).
At matched 3-step inference, the isolated $90\%\!\to\!97\%$ sparsity stretch contributes $1.31$--$1.32\times$; this is the regime where step-local-only recipes degrade (Sec.~\ref{sec:trap}).

\begin{table}[t]
\centering
\scriptsize

\caption{
Comparison of generation quality and efficiency on VBench~\citep{huang2024vbench} and VBench-2.0~\citep{zheng2025vbench}.
We evaluate Wan2.1-T2V-1.3B at $480\times832$, and Wan2.1-T2V-14B and Wan2.2-T2V-A14B at $720\times1280$, all with 81 frames.
VBench-2.0 reports the overall score and five capability categories: Creativity, Commonsense, Controllability, Human Fidelity, and Physics.
``Sp.'' denotes attention sparsity; NFE conventions follow footnotes~\ref{fn:fn1}.
All \method rows are measured with the full deployment stack active---W8A8 FP8 quantization with fused kernels---for both quality metrics and latency; baselines are evaluated with their official inference pipelines, including their own quantization settings. Full Attention rows use our strongest dense implementation: BF16 with FlashAttention-2 on RTX~5090 and FlashAttention-3 on H100, and no other acceleration technique.
For Wan2.2-T2V-A14B-720P, the RTX 5090 latency includes the overhead of swapping the high-noise and low-noise expert groups in GPU memory; on H100 both expert groups reside in memory simultaneously.
\textbf{Bold} marks the best result among accelerated methods within each model block.
}

\label{tab:main_vbench}

\begin{tblr}{
width=\linewidth,
colspec={Q[c,wd=2.00cm] Q[c,wd=3.20cm] *{7}{X[c]} Q[c,wd=0.60cm] Q[c,wd=0.65cm] Q[c,wd=0.65cm]},
colsep=1.8pt,
rows={rowsep=2.2pt},
row{1,2}={bg=SparkDiffusionHeader,font=\bfseries},
cell{3}{2-12}={fg=gray},
cell{7}{2-12}={fg=gray},
cell{12}{2-12}={fg=gray},
cell{6}{2-12}={bg=SparkDiffusionLight},
cell{10,11}{2-12}={bg=SparkDiffusionLight},
cell{13,14}{2-12}={bg=SparkDiffusionLight},
cell{6,10,11,13,14}{2}={fg=SparkDiffusionRed,font=\bfseries},
hline{1}={1.25pt,SparkDiffusionRed},
hline{2}={3}{0.7pt,SparkDiffusionRed,r=0.75},
hline{2}={4-9}{0.7pt,SparkDiffusionRed,l=0.75},
hline{2}={11-12}{0.7pt,SparkDiffusionRed},
hline{3}={0.7pt,SparkDiffusionRed},
hline{7}={0.7pt,SparkDiffusionOrange},
hline{12}={0.7pt,SparkDiffusionOrange},
hline{15}={1.25pt,SparkDiffusionRed}
}

& & \SetCell[c=1]{c} VBench & \SetCell[c=6]{c} VBench-2.0 & & & & & & &  \SetCell[c=2]{c} Latency (s) & \\

Model
& Method
& Total $\uparrow$
& Total $\uparrow$
& Creat. $\uparrow$
& Common. $\uparrow$
& Control. $\uparrow$
& {\textbf{Human}\\[-1pt]\textbf{Fid.} $\uparrow$}
& Physics $\uparrow$
& Sp.
& 5090 $\downarrow$
& H100 $\downarrow$
\\

\SetCell[r=4]{c,m}{
\textbf{Wan2.1-T2V}\\[-1pt]
\textcolor{SparkDiffusionGray}{\footnotesize 1.3B}
}
& Full
& 83.21
& 56.02
& 54.73
& 57.38
& 34.96
& 78.71
& 54.30
& $0\%$
& 182
& 92
\\

& FastWan (VSA) & 82.37 & 54.63 & 52.66 & 56.52 & 32.27 & 79.69 & 52.01 & $90\%$ & 2.8 & 1.2 \\

& TurboDiffusion & {82.52} & 54.61 & 51.72 & 56.94 & 31.65 & 79.55 & 53.21 & $90\%$ & 2.0 & 1.0 \\

& SparkDiffusion (Ours) & \best{82.64} & \best{55.85} & \best{53.99} & \best{57.13} & \best{33.75} & \best{80.34} & \best{54.02} & $90\%$ & \best{1.3} & \best{0.6} \\

\SetCell[r=5]{c,m}{
\textbf{Wan2.1-T2V}\\[-1pt]
\textcolor{SparkDiffusionGray}{\footnotesize 14B}
}
& Full
& 83.69
& 60.20
& 55.25
& 63.98
& 37.32
& 81.60
& 62.84
& $0\%$
& 4769
& 1757
\\

& FastWan (VSA) & 82.72 & 58.04 & 54.97 & 59.54 & 35.75 & 81.72 & 58.20 & $90\%$ & 54.1 & 20.5 \\

& TurboDiffusion & 82.88 & 57.98 & 54.88 & 59.77 & 34.98 & 81.51 & 58.74 & $90\%$ & 25.3 & 16.0 \\

& SparkDiffusion (Ours) & \best{83.42} & \best{59.36} & \best{55.22} & \best{60.88} & \best{37.41} & \best{83.51} & \best{59.77} & $90\%$ & 23.7 & 10.5 \\

& SparkDiffusion (Ours) & 83.15 & 58.05 & 54.81 & 59.02 & 36.01 & 81.77 & 58.63 & $97\%$ & \best{18.0} & \best{8.0} \\

\SetCell[r=3]{c,m}{
\textbf{Wan2.2-T2V}\\[-1pt]
\textcolor{SparkDiffusionGray}{\footnotesize A14B}
}
& Full
& 84.21
& 60.36
& 55.60
& 64.50
& 37.40
& 81.30
& 63.00
& $0\%$
& 4545
& 1508
\\
& SparkDiffusion (Ours) & \best{83.75} & \best{59.77} & \best{55.51} & \best{61.36} & \best{37.55} & \best{83.88} & \best{60.53} & $90\%$ & 32.4 & 10.5 \\
& SparkDiffusion (Ours) & 83.36 & 58.46 & 55.09 & 59.32 & 35.67 & 81.92 & 60.30 & $97\%$ & \best{25.1} & \best{8} \\

\end{tblr}

\end{table}

\subsection{Diversity preservation under few-step distillation}
\label{subsec:exp_diversity}

\begin{table}[t]

\centering
\scriptsize

\caption{
Quantitative diversity comparison on Wan2.1-T2V-14B-720P.
Each video is encoded by two frozen video encoders, V-JEPA 2~\citep{assran2025v} and VideoMAE V2~\citep{wang2023videomaev2}; we report the average pairwise cosine and $\ell_2$ distances among $N{=}5$ videos sampled from the same prompt with 5 noise seeds ($1{,}000$ prompts; protocol follows~\citet{shaul2026pdd}).
}
\label{tab:diversity}

\begin{tblr}{
width=\linewidth,
colspec={Q[l,wd=3.25cm] Q[c,wd=0.55cm] *{4}{X[c]}},
colsep=1.8pt,
rows={rowsep=2.2pt},
row{1,2}={bg=SparkDiffusionHeader,font=\bfseries},
cell{3}{1-6}={fg=gray},
cell{6}{1-6}={bg=SparkDiffusionLight},
cell{6}{1}={fg=SparkDiffusionRed,font=\bfseries},
hline{1}={1.25pt,SparkDiffusionRed},
hline{2}={3-4}{0.7pt,SparkDiffusionRed,r=0.75},
hline{2}={5-6}{0.7pt,SparkDiffusionRed,l=0.75},
hline{3}={0.7pt,SparkDiffusionRed},
hline{7}={1.25pt,SparkDiffusionRed}
}

& &
\SetCell[c=2]{c} V-JEPA 2
& &
\SetCell[c=2]{c} VideoMAE V2
& \\

Method
& Sp.
& Cos $\uparrow$
& $\ell_2$ $\uparrow$
& Cos $\uparrow$
& $\ell_2$ $\uparrow$
\\

Full Attention
& $0\%$
& 0.125
& 27.15
& 0.0252
& 2.83
\\

FastWan (VSA)
& $90\%$
& 0.075
& 21.31
& 0.0117
& 2.05
\\

TurboDiffusion
& $90\%$
& 0.078
& 21.67
& 0.0125
& 2.21
\\

\mbox{SparkDiffusion (Ours)}
& $97\%$
& \best{0.087}
& \best{23.04}
& \best{0.0142}
& \best{2.47}
\\

\end{tblr}
\end{table}

Table~\ref{tab:diversity} quantifies seed-level diversity, complementing the qualitative check in Figure~\ref{fig:diversity_three_methods}.
Following~\citet{shaul2026pdd}, we encode each video with two frozen video encoders, V-JEPA 2~\citep{assran2025v} and VideoMAE V2~\citep{wang2023videomaev2}, and report average pairwise cosine and $\ell_2$ distances among videos sampled from the same prompt with different noise seeds; higher values indicate more diverse samples, and the dense model serves as a reference.
Despite operating at $97\%$ sparsity, \method stays closest to the dense reference, whereas the few-step baselines lose a visible fraction of seed-level variation, consistent with the mode-seeking tendency of distribution matching~\citep{chen2026data}.
This supports the trajectory-mixed design of Sec.~\ref{subsec:distill}: high-noise consistency matching preserves coarse structure and diversity, while low-noise distribution matching improves terminal fidelity.

\subsection{Distillation objective at extreme sparsity}
\label{subsec:exp_objective_ablation}
\begin{wraptable}{r}{0.53\linewidth}
    
    \centering
    \scriptsize

    \caption{
    Distillation-objective ablation at $97\%$ attention sparsity
    (Wan2.1-T2V-14B-720P; 3-step CFG-free student; All 3-step student rows initialize from the same Stage-1 warm-up checkpoint.)
    }
    \label{tab:objective_ablation}

    \begin{tblr}{
        width=\linewidth,
        colspec={Q[l,wd=4.25cm] X[c] X[c]},
        colsep=1.8pt,
        rows={rowsep=2.2pt},
        row{1}={bg=SparkDiffusionHeader,font=\bfseries},
        cell{2}{1-3}={fg=gray},
        cell{5}{1-3}={bg=SparkDiffusionLight},
        cell{5}{1}={fg=SparkDiffusionRed,font=\bfseries},
        hline{1}={1.25pt,SparkDiffusionRed},
        hline{2}={0.7pt,SparkDiffusionRed},
        hline{6}={1.25pt,SparkDiffusionRed}
    }

    Stage-2 objective
    & VBench $\uparrow$
    & VBench-2.0 $\uparrow$
    \\

    \mbox{Full Attention (dense, 50-step)}
    & 83.69
    & 60.20
    \\

    PCM only (3-step)
    & 81.94
    & 56.41
    \\

    DMD only (3-step)
    & 82.56
    & 57.38
    \\

    \mbox{CrossDistill (Ours, 3-step)}
    & \best{83.15}
    & \best{58.05}
    \\

    \end{tblr}

    \vspace{-10pt}
\end{wraptable}

Table~\ref{tab:objective_ablation} ablates the Stage-2 objective at $97\%$ sparsity, with all students initialized from the same Stage-1 warm-up checkpoint and sampled with three CFG-free steps.
Both single-objective variants already recover most of the quality lost at $97\%$ sparsity, corroborating that the trap is one of supervision: any terminal-aligned objective mitigates it.
The two pure objectives, however, fail in opposite directions along the fidelity--diversity axis: PCM-only training tracks the teacher's trajectory but never aligns the terminal distribution, yielding the weakest scores, while DMD-only training aligns the terminal distribution but is mode-seeking, losing coarse structure and seed-level diversity (cf.\ Table~\ref{tab:diversity}).
The trajectory-mixed objective attains the best of both segments; together with the warm-up ablation (Fig.~\ref{fig:warm_up_ablation}), this supports the staging principle of Sec.~\ref{sec:trap}.

\begin{figure}[htbp]
    \centering
    \includegraphics[width=0.75\linewidth]{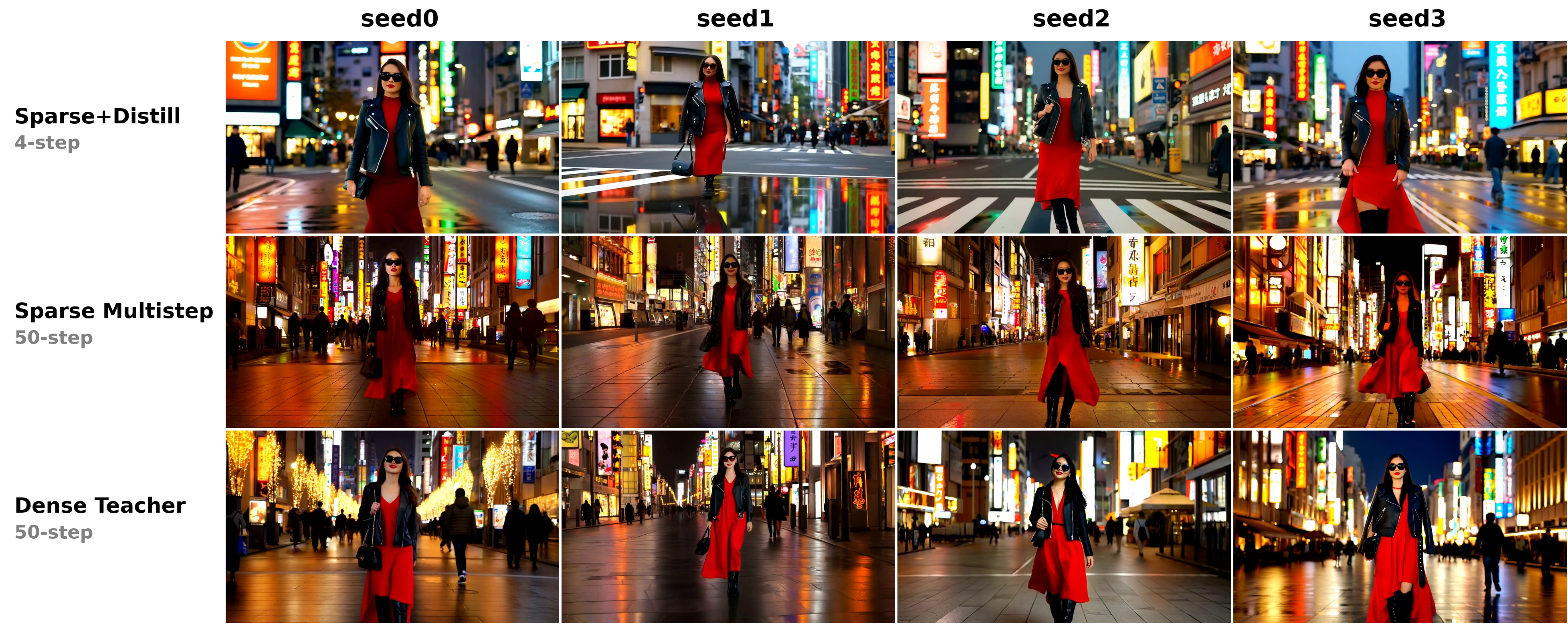}
\caption{
         Few-step distillation preserves sample diversity.
        On Wan2.1-T2V-14B-480P, for each of two prompts, we show one frame per video under three models
        (dense teacher; sparse model with 90\% block sparsity and rank-64 compensation; 3-step student distilled from it)
        using the same four noise seeds.
}
\label{fig:diversity_three_methods}
\end{figure}
\subsection{Cross-model qualitative validation}
\label{subsec:exp_cross_model}

To examine whether \method{} generalizes beyond the primary benchmark setting, we conduct qualitative comparisons across multiple model scales, tasks, resolutions, and sparsity levels:
Wan2.1-T2V-1.3B at $480\times832$ with $90\%$ sparsity (Figure~\ref{fig:wan21_1.3b_480p_dense_turbo_ours}),
Wan2.1-T2V-14B at $480\times832$ with $90\%$ sparsity (Figure~\ref{fig:wan21_14b_480p_dense_turbo_ours}),
Wan2.1-T2V-14B at $720\times1280$ with $97\%$ sparsity (Figure~\ref{fig:wan21_14b_720p_dense_turbo_ours}),
Wan2.1-I2V-14B at $720\times1280$ with $97\%$ sparsity (Figure~\ref{fig:wan21_14b_i2v720p_full_vs_sparse_top005}),
and Wan2.2-T2V-A14B at $720\times1280$ with $97\%$ sparsity (Figure~\ref{fig:app_wan22_t2v_720p_full_vs_ours_97}).
All accelerated models use 3-step inference; TurboDiffusion and FastWan use their officially released weights and inference scripts.
Full settings are given in Appendix~\ref{app:additional_qualitative}.
These results suggest that \method{} maintains coherent structure, semantic alignment, and temporal detail even when attention sparsity is pushed to $97\%$.

\section{Conclusion}
\label{sec:conclusion}

We identified the \emph{high-sparsity trap}: at extreme attention sparsity, step-local training converges while terminal generation quality degrades, and its root cause lies in the supervision paradigm.
The following staging principle---\emph{first adapt the sparse architecture into a coarse prior, then correct the terminal distribution}---naturally leads to \method, a unified acceleration framework for visual generation chaining a short sparse warm-up, few-step trajectory-mixed distillation, and FP8 quantization with fused kernels.
\method sustains $97\%$ attention sparsity on long-sequence 720P generation with strong visual quality, delivering a measured $265\times$ end-to-end speedup on Wan2.1-T2V-14B-720P on a single RTX~5090 ($220\times$ on H100) and generating a Wan2.1-T2V-1.3B-480P video end-to-end in $1.3$ seconds, across Wan2.1/Wan2.2 backbones, T2V/I2V tasks, and 480P/720P resolutions.
Going forward, we will extend this framework to bring high-sparsity acceleration to omni-modal generative models and autoregressive world models.

\section{Acknowledgments}
This work was supported by Alibaba Group through Alibaba Research Intern Program.

\clearpage
\bibliographystyle{plainnat}
\bibliography{references}

\clearpage
\beginappendix

\section{Related Work}
\label{sec:related}

\paragraph{High-sparsity attention for video diffusion.}
Sparse attention is a widely used route for reducing the quadratic attention cost of video DiTs. Training-free methods exploit recurring spatiotemporal structure, search structured sparse patterns, or construct sparse masks dynamically during inference~\citep{xi2025sparse,yang2026sparse,chen2026sparse,liu2026mixture,zhang2025spargeattention,li2026radial,zhang2025fast,xia2025training}, while trainable approaches learn adaptive sparse computation from data~\citep{zhang2026faster,wu2026vmoba}. To retain global context under aggressive sparsification, compensation-based methods such as SLA~\citep{zhang2026sla}, SLA2~\citep{zhang2026sla2}, and SALAD~\citep{fang2026salad} combine sparse attention with lightweight linear branches. RoPeSLR~\citep{liu2026ropeslr} further identifies a 3D-RoPE-aware sparse-plus-low-rank structure in video attention, where high-energy semantic interactions are captured by a sparse branch and the remaining background context is modeled by a low-rank MLP. These methods motivate the architectural feasibility of high sparsity, but they primarily focus on sparse-plus-compensation designs and do not directly address the terminal-quality failure that can appear when sparsity is pushed to the extreme regime studied in this paper.

\paragraph{Sparse attention and few-step distillation.}
Few-step diffusion generation has been widely explored through distribution-matching and consistency-based approaches~\citep{yin2024one,yin2024improved,song2023consistency,wang2024phased,luo2025learning,zheng2026large}. Recent works combine sparse attention with few-step distillation to reduce both per-step computation and sampling steps. VSA~\citep{zhang2026faster} jointly trains sparse attention with DMD-style distillation, but DMD-style objectives are known to reduce sample diversity due to mode-seeking behavior~\citep{chen2026data}.  TurboDiffusion~\citep{zhang2025turbodiffusion} trains its
sparse-attention (SLA)\cite{zhang2026sla} and consistency (rCM)\cite{zheng2026large} components separately and then merges the
parameter updates.
\method instead takes a staged route motivated by the diagnosis in Sec.\ref{sec:trap}:
sparse architecture adaptation and terminal alignment address two distinct failure
modes---the former produces a usable generative prior under extreme sparsity, the
latter corrects the terminal distribution, so we decouple them into two sequential
phases, each with a single standard objective.

\section{Additional Results}
\label{app:Additional_Results}

\subsection{Additional Qualitative Results}
\label{app:additional_qualitative}

This appendix reports additional qualitative results across model scales, tasks, resolutions, and attention-sparsity levels.
Full Attention is evaluated with 50 sampling steps and classifier-free guidance (NFE${}=100$); all accelerated models are evaluated with three CFG-free inference steps (NFE=3).
Where included, TurboDiffusion and FastWan are reproduced from their officially released weights and inference scripts.
The same prompts, random seeds, and displayed temporal positions are used across all comparisons.

\begin{figure}[htbp]
    \centering
    \includegraphics[width=\linewidth]{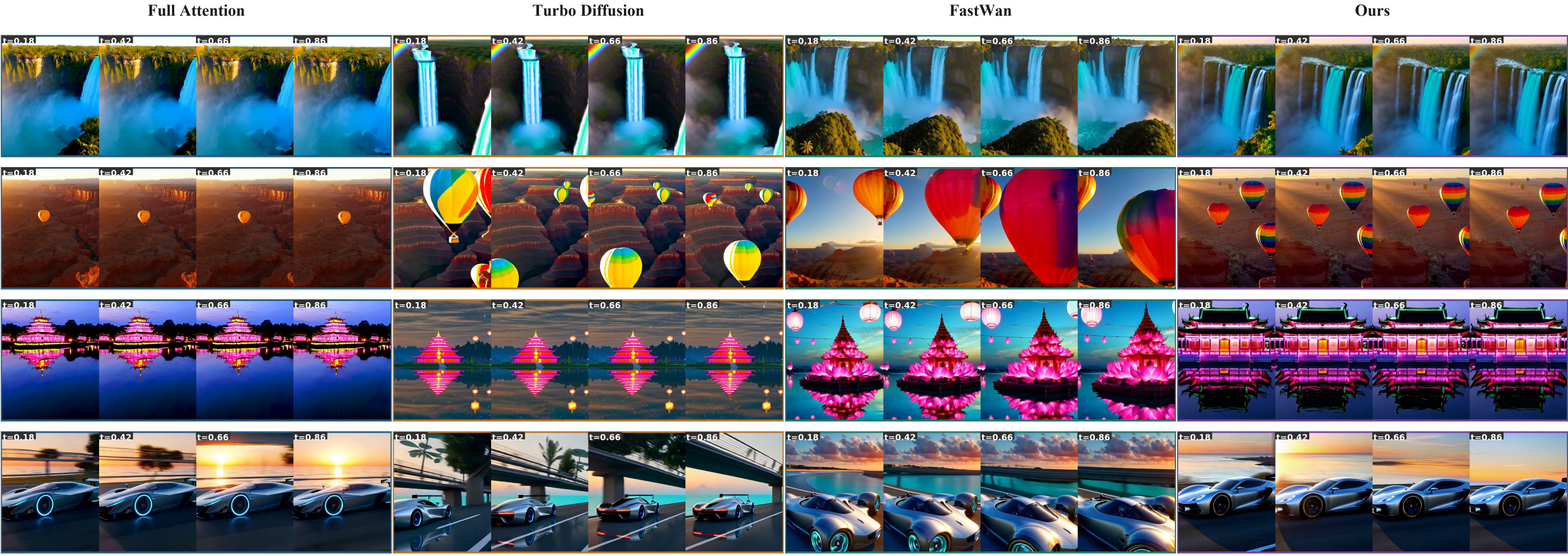}
    \caption{
        \textbf{Wan2.1-T2V-1.3B-480P.}
        Qualitative results at $480\times832$ resolution with 81 frames; TurboDiffusion and
FastWan and SparkDiffusion operate at $90\%$ attention sparsity.
    }
    \label{fig:wan21_1.3b_480p_dense_turbo_ours}
\end{figure}

\begin{figure}[htbp]
    \centering
    \includegraphics[width=\linewidth]{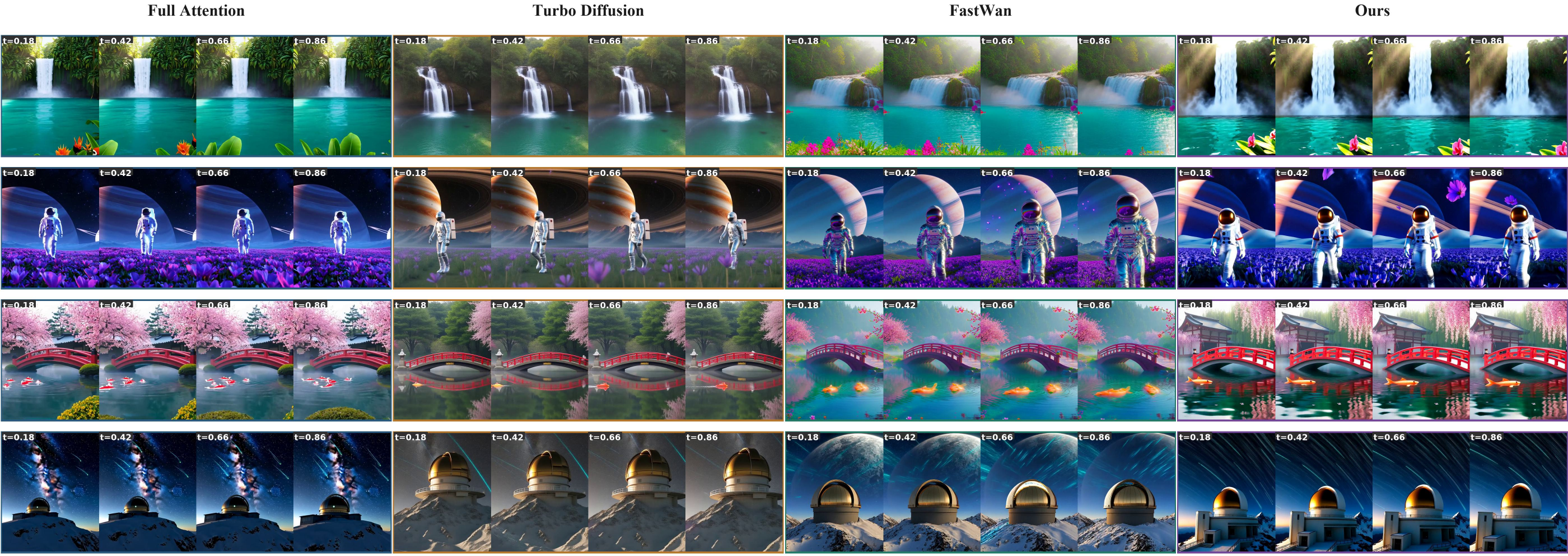}
    \caption{
        \textbf{Wan2.1-T2V-14B-480P.}
        Qualitative results at $480\times832$ resolution with 81 frames; all methods operate at $90\%$ attention sparsity.
    }
    \label{fig:wan21_14b_480p_dense_turbo_ours}
\end{figure}

\begin{figure}[htbp]
    \centering
    \includegraphics[width=\linewidth]{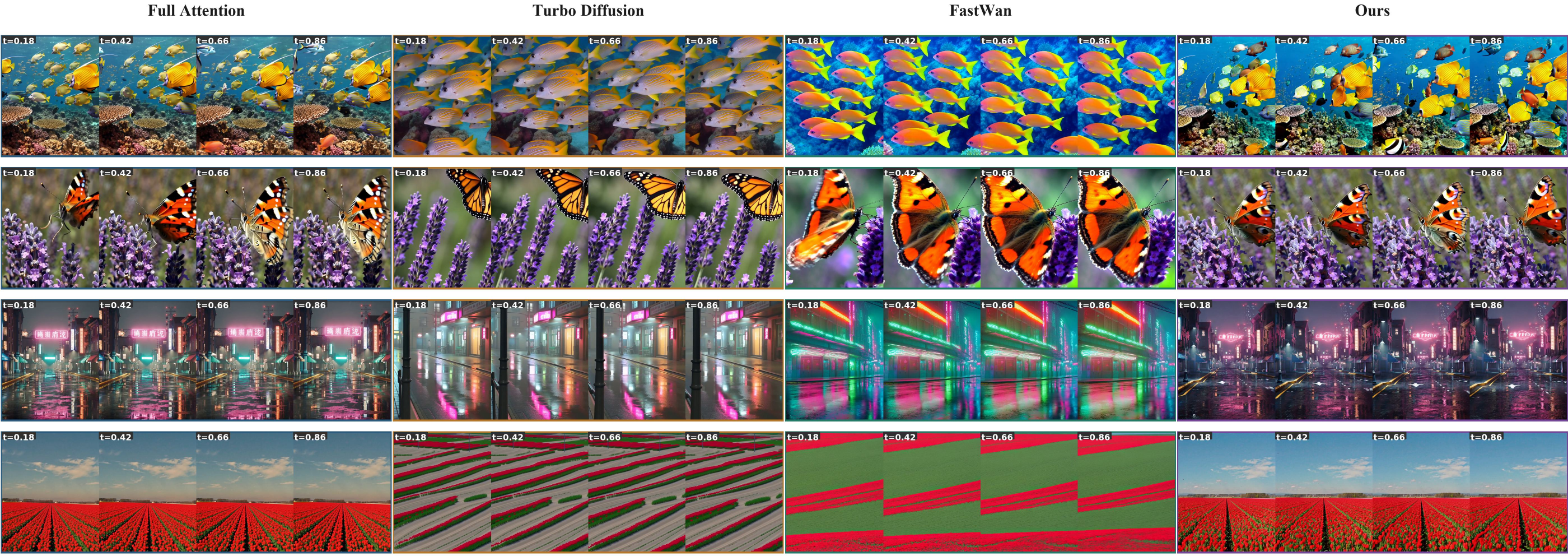}
    \caption{
        \textbf{Wan2.1-T2V-14B-720P.}
        Qualitative results at $720\times1280$ resolution with 81 frames; TurboDiffusion and FastWan operate at $90\%$ attention sparsity, \method{} at $97\%$.
    }
    \label{fig:wan21_14b_720p_dense_turbo_ours}
\end{figure}

\begin{figure}[htbp]
    \centering
    \includegraphics[width=0.85\linewidth]{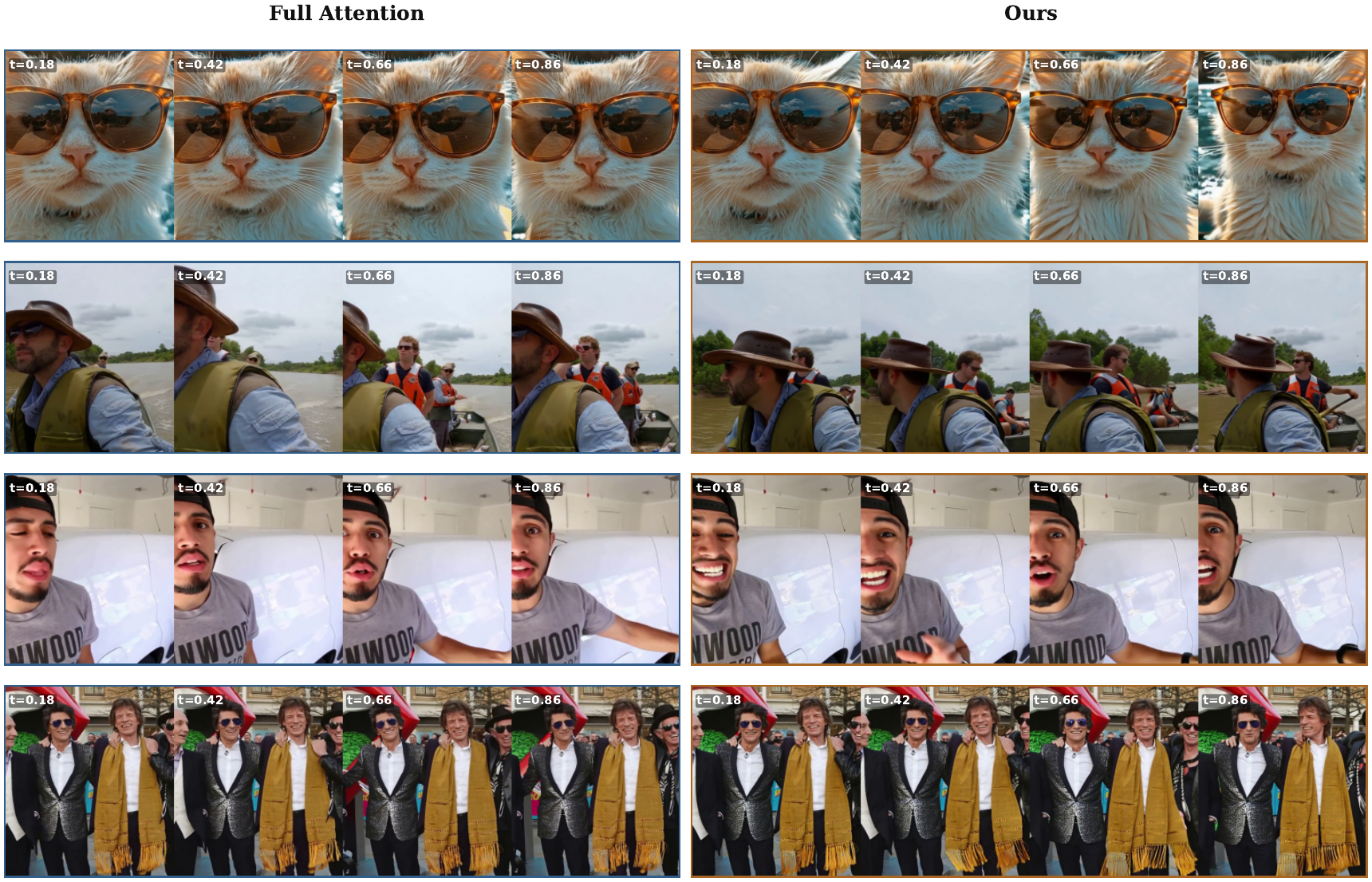}
    \caption{
        \textbf{Wan2.1-I2V-14B-720P.}
        Qualitative results at $720\times1280$ resolution with 81 frames; \method{} operates at $97\%$ attention sparsity.
    }
    \label{fig:wan21_14b_i2v720p_full_vs_sparse_top005}
\end{figure}

\begin{figure}[htbp]
    \centering
    \includegraphics[width=0.85\linewidth]{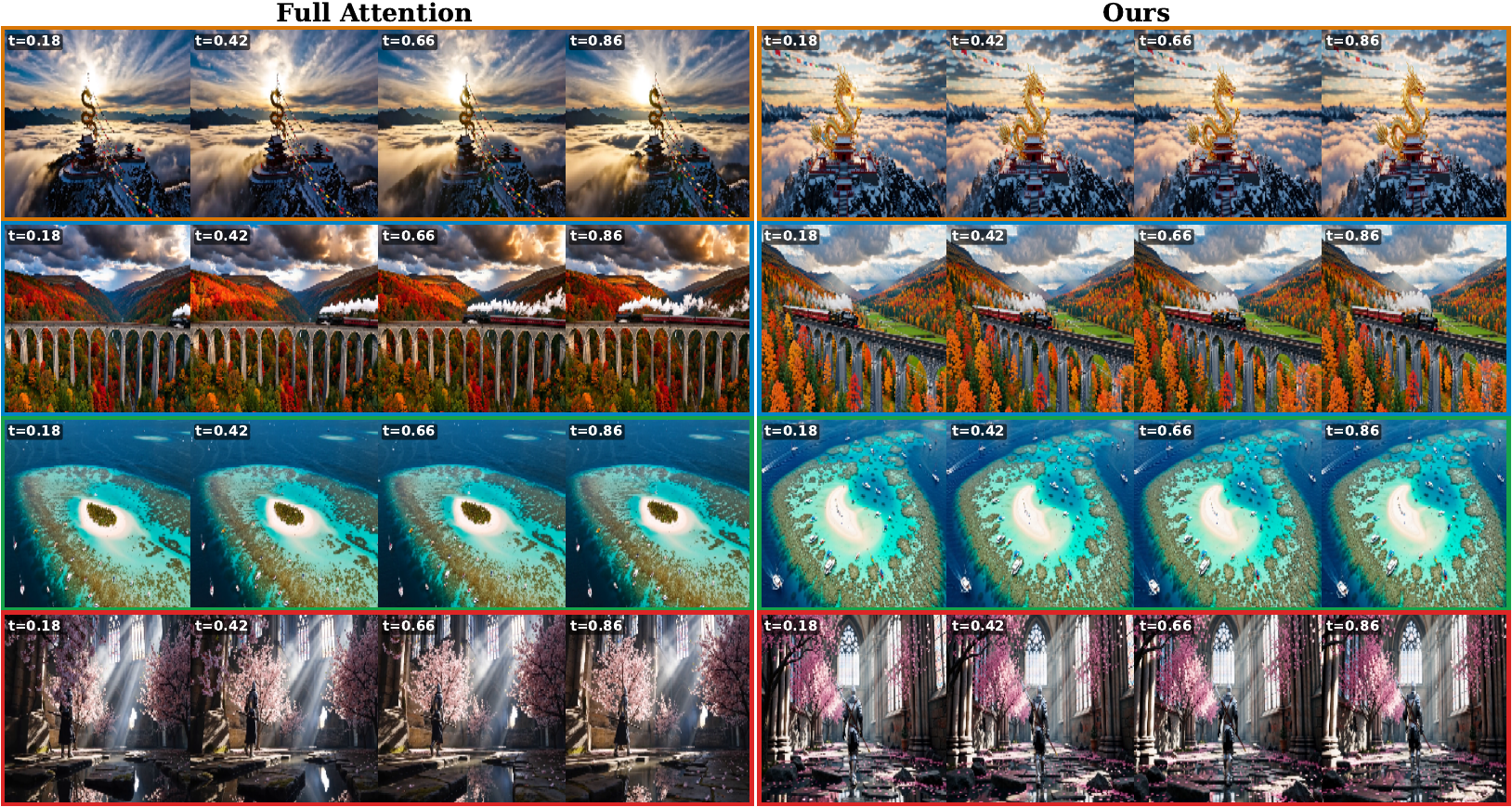}
    \caption{
        \textbf{Wan2.2-T2V-A14B-720P.}
        Qualitative results at $720\times1280$ resolution with 81 frames; \method{} operates at $97\%$ attention sparsity.
    }
    \label{fig:app_wan22_t2v_720p_full_vs_ours_97}
\end{figure}

\begin{figure}[H]
    \centering
    \includegraphics[width=0.5\linewidth]{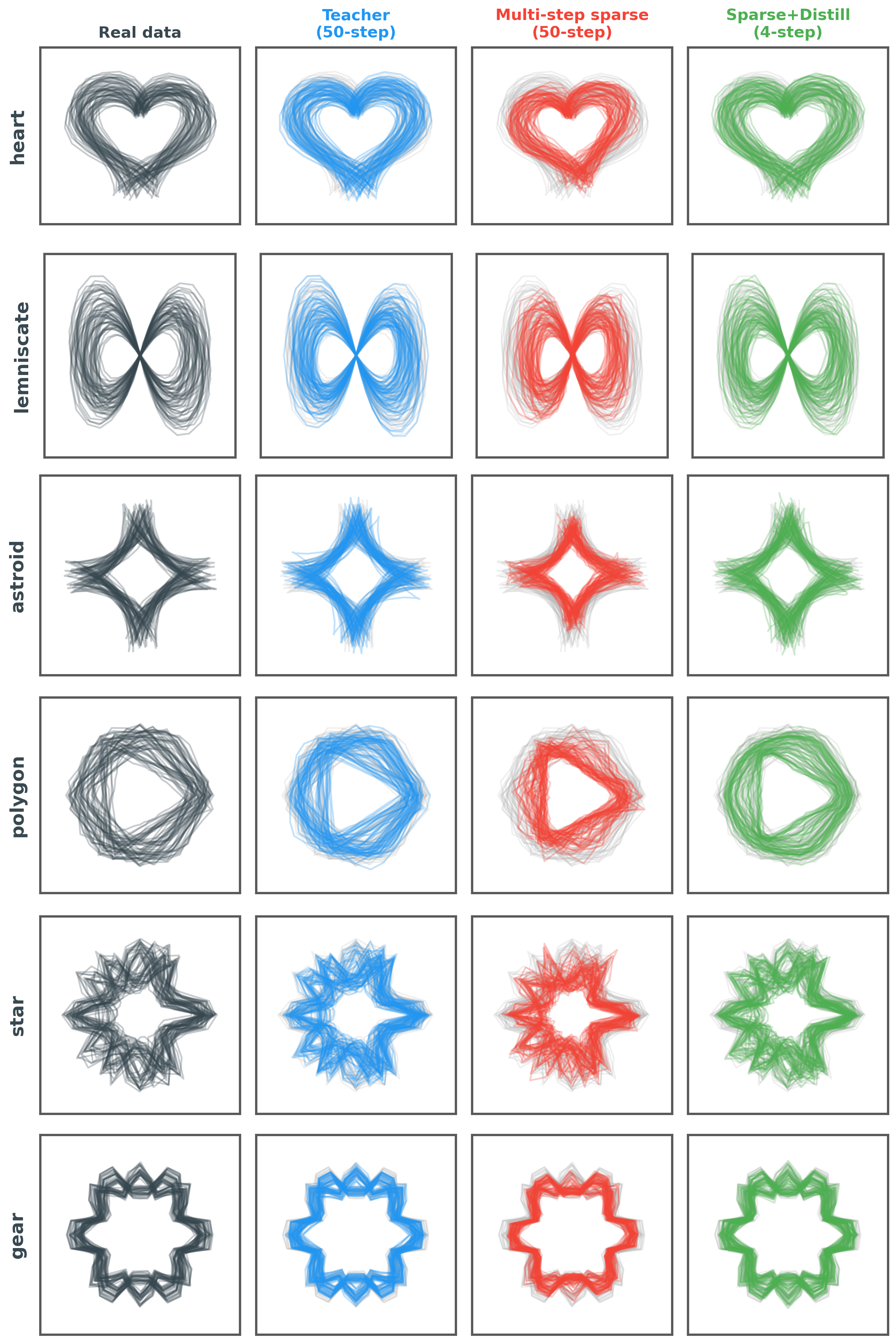}
\caption{
        Qualitative sample overlays for the same six toy sequence manifolds.
        In this toy setting, the multi-step sparse model can drift away from the data distribution,
        producing distorted or misaligned contours,
        while the 3-step distilled student produces more globally consistent shapes.
}
    \label{fig:mainfld_showcase}
\end{figure}

\subsection{FP8 versus BF16 quality ablation}
\label{app:fp8_ablation}

Table~\ref{tab:fp8_ablation} reports the quality impact of the W8A8 FP8 deployment step of Sec.~\ref{subsec:fp8}, comparing each deployed model against its BF16 counterpart under identical 3-step CFG-free inference.
Across backbones and sparsity levels, FP8 quantization changes VBench and VBench-2.0 by at most $0.07$ points, confirming that with few-step sampling the quantization error has little room to accumulate and that the speedups in Fig.~\ref{fig:sparkdiffusion_latency_speedup_h100_rtx5090} come at negligible quality cost.
All \method quality numbers in the main text are therefore reported directly on the deployed FP8 model.

\begin{table}[htbp]
\centering
\small
\setlength{\tabcolsep}{6pt}
\renewcommand{\arraystretch}{1.15}
\caption{
FP8 (W8A8) versus BF16 generation quality on deployed \method models (3-step CFG-free inference).
}
\label{tab:fp8_ablation}
\arrayrulecolor{SparkDiffusionRed}
\begin{tabular}{llcccc}
\toprule[1.25pt]
\rowcolor{SparkDiffusionHeader}
\textbf{Model} & \textbf{Sp.} & \textbf{VBench (BF16)} & \textbf{VBench (FP8)} & \textbf{VBench-2.0 (BF16)} & \textbf{VBench-2.0 (FP8)} \\
\midrule[0.65pt]
Wan2.1-T2V-1.3B-480P & $90\%$ & 82.68 & 82.64 & 55.91 & 55.85 \\ 
Wan2.1-T2V-14B-720P  & $90\%$ & 83.47 & 83.42 & 59.41 & 59.36 \\ 
Wan2.1-T2V-14B-720P  & $97\%$ & 83.22 & 83.15 & 58.12 & 58.05 \\ 
Wan2.2-T2V-A14B-720P & $97\%$ & 83.41 & 83.36 & 58.53 & 58.46 \\ 
\bottomrule[1.25pt]
\end{tabular}
\arrayrulecolor{black}
\end{table}


\makeatletter
\@ifundefined{theorem}{\newtheorem{theorem}{Theorem}[section]}{}
\@ifundefined{proposition}{\newtheorem{proposition}[theorem]{Proposition}}{}
\@ifundefined{lemma}{\newtheorem{lemma}[theorem]{Lemma}}{}
\@ifundefined{corollary}{\newtheorem{corollary}[theorem]{Corollary}}{}
\@ifundefined{definition}{\newtheorem{definition}[theorem]{Definition}}{}
\@ifundefined{assumption}{\newtheorem{assumption}[theorem]{Assumption}}{}
\@ifundefined{example}{\newtheorem{example}[theorem]{Example}}{}
\@ifundefined{remark}{\newtheorem{remark}[theorem]{Remark}}{}
\@ifundefined{limbox}{%
  \newenvironment{limbox}{\par\smallskip\noindent\ignorespaces}{\par\smallskip}}{}
\@ifundefined{limbox@blue}{}{}
\makeatother


\newpage
\section{Idealized Mechanism: Why Step-Local Training Can Leave a Terminal Error}
\label{sec:insight}
\label{sec:insight_terminal_misalignment}

\mathtoolsset{showonlyrefs=false}
\renewcommand{\eqref}[1]{\textup{(\ref{#1})}}

\begin{limbox_blue}
\textbf{In three sentences.}
The flow-matching loss sums the \emph{squared magnitudes} of per-step velocity errors, whereas the terminal sample depends on a \emph{signed, transported sum} of the same errors along the sampling trajectory.
Therefore two models can have identical step-local training loss but very different terminal outputs: one whose per-step errors point consistently and accumulate, and one whose errors oscillate and cancel.
In a deliberately small surrogate, we show that the terminal-visible error left by a step-locally optimal model can still be reduced through the \emph{same} trainable parameters by a terminal-aligned signal, and we compute the surrogate price of doing so.
\end{limbox_blue}

\subsection{Roadmap}
\label{ins:sec:roadmap}

Figure~\ref{fig:ins-roadmap} summarizes the logical flow of the appendix.

\begin{figure}[t]
\centering
\begin{tikzpicture}[
  >=Latex,
  insbox/.style={
    draw,
    rounded corners=3pt,
    align=left,
    inner sep=5pt,
    text width=5.2cm,
    minimum height=2.4cm,
    anchor=north,
    font=\footnotesize
  },
  setupbox/.style={insbox, draw=black!55, fill=black!4},
  stepbox/.style={insbox, draw=blue!45!black, fill=blue!5}
]

\node[setupbox] (setup) at (0,0)
{
  \textbf{Surrogate setup}\\[2pt]         
  Euler sampler, affine reference field, affine terminal readout, affine trainable residual
};

\node[stepbox] (step1) at (6.0,0)
{
  \textbf{Step 1} (Lemma~\ref{ins:lem:pairing})\\[2pt]
  Terminal error equals the sensitivity-weighted sum of per-step errors:
  $\Psi=\langle e,s\rangle_h$
};

\node[stepbox] (step2) at (0,-2.9)
{
  \textbf{Step 2} (Proposition~\ref{ins:prop:accumulation})\\[2pt]
  Coherent errors accumulate over the sampling horizon; oscillatory errors can cancel at the same step-local loss
};

\node[stepbox] (step3) at (6.0,-2.9)
{
  \textbf{Step 3} (Proposition~\ref{ins:prop:survive})\\[2pt]
  Step-local training removes the reachable part of the error and stops; the leftover is invisible to the step-local gradient but visible terminally
};

\node[stepbox] (step4) at (3.0,-5.8)
{
  \textbf{Step 4} (Corollary~\ref{ins:cor:price}, Proposition~\ref{ins:prop:terminal_descent})\\[2pt]
  A terminal-aligned signal reduces the surviving terminal error through the same trainable subspace, at an explicitly computed surrogate price
};

\draw[->, thick, black!65] (setup.east) -- (step1.west);
\draw[->, thick, black!65] (setup.south) -- (step2.north);
\draw[->, thick, black!65] (step1.south) -- (step3.north);
\draw[->, thick, black!65] (step2.south) -- ++(0,-0.25) -| (step4.north);
\draw[->, thick, black!65] (step3.south) -- ++(0,-0.25) -| (step4.north);

\end{tikzpicture}
\caption{
\textbf{Roadmap of the mechanism.}
Each box states one step of the argument and where it is proved.
The proof is Hilbert-space geometry: a quadratic step-local loss and a linear terminal functional see the same residual differently.
Sparsity enters only through empirical inputs---the residual left by sparse training and the trainable correction subspace---not as a mathematical object in the proof.
}
\label{fig:ins-roadmap}
\end{figure}
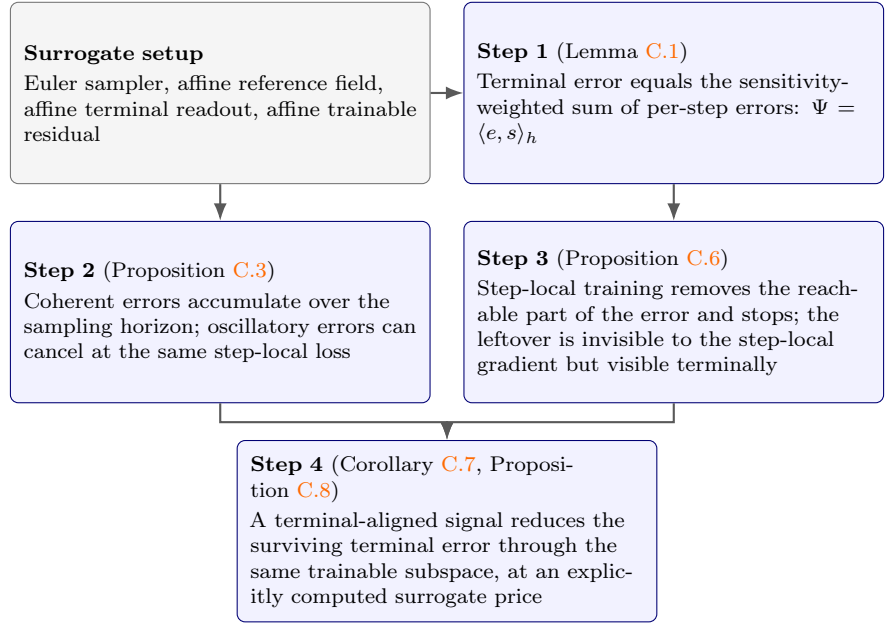

The argument has four steps:
\begin{enumerate}[label=\textbf{Step \arabic*.},leftmargin=4.4em,itemsep=0.25em]
\item The terminal error is exactly the sensitivity-weighted sum of per-step velocity errors.
\item If the per-step errors are coherently aligned with the terminal sensitivity, they accumulate; if they oscillate, they can cancel, while the step-local loss remains unchanged.
\item Step-local training removes exactly the part of the error that lies in the trainable subspace and leaves the orthogonal leftover. At the step-local optimum, the loss gradient vanishes, but the terminal error need not vanish.
\item A terminal-aligned objective can reduce this surviving terminal error using the same trainable subspace. The first-order terminal gain costs only second-order step-local sub-optimality.
\end{enumerate}

\begin{limbox}
\textbf{What sparsity does and does not do here.}
Sparsity is \emph{not} a mathematical object in this appendix.
The analysis takes as input that high-sparsity training leaves a residual that is
(i) nonzero,
(ii) coherently aligned with the terminal sensitivity rather than zero-mean jitter, and
(iii) partly outside the span of the trainable correction directions.
We do not prove that aggressive sparsification produces these properties.
The conclusion is conditional:
\emph{if} high-sparsity training leaves such a residual,
\emph{then} step-local training cannot remove its terminal-visible part, whereas a terminal-aligned signal can reduce it at first order through the same trainable parameters.
\end{limbox}

\subsection{Setup and notation}
\label{ins:sec:setup}

We use a deliberately small surrogate in which every object has a simple geometric meaning.
Table~\ref{tab:insight-notation} summarizes the notation.

\begin{table}[t]
\centering
\small
\caption{
Notation used in the appendix.
The key objects are the per-step velocity error $e$, the terminal sensitivity $s$, the trainable correction subspace $\mathcal U$, and the two objectives: the step-local loss $\mathcal L$ and the terminal error $\Psi$.
}
\label{tab:insight-notation}
\begin{tabular}{@{}p{0.19\textwidth} p{0.73\textwidth}@{}}
\toprule
Symbol & Meaning \\
\midrule
$N,h,T$ & number of sampler steps, step size, horizon $T=Nh$ \\
$x_k^\star$ & reference trajectory generated by $v_\star$ \\
$e_k$ & per-step velocity error at step $k$ \\
$e=(e_0,\dots,e_{N-1})$ & full velocity-error sequence \\
$R(x)=c^\top x$ & linear terminal observable \\
$\Psi$ & terminal error $R(x_N)-R(x_N^\star)$ \\
$s_k$ & terminal sensitivity transported back to step $k$ \\
$\langle\cdot,\cdot\rangle_h$ & step-weighted inner product $h\sum_k a_k^\top b_k$ \\
$\mathcal U$ & subspace of residual corrections reachable by training \\
$P_{\mathcal U}$, $P_{\mathcal U^\perp}$ & orthogonal projections onto $\mathcal U$ and its orthogonal complement \\
$s_{\mathcal U}$, $s_{\mathcal U^\perp}$ & reachable and unreachable parts of the terminal sensitivity \\
$\ell$ & leftover residual $P_{\mathcal U^\perp}e^0$ after step-local training \\
$\beta$ & norm of the leftover residual, $\beta=\|\ell\|_h$ \\
$\mu,\epsilon$ & coherence constant and per-step error floor \\
$\mathcal L$ & step-local quadratic loss $\frac12\|e\|_h^2$ \\
$\Delta$ & step-local sub-optimality relative to the step-local optimum \\
$\mathcal C$ & surrogate price of cancelling the terminal error \\
\bottomrule
\end{tabular}
\end{table}

\paragraph{Sampler, reference field, and terminal observable.}
Fix $N\in\mathbb N$, step size $h>0$, horizon $T=Nh$, and times $t_k=kh$.
The index $k$ runs along the sampling recursion; if the deployed sampler traverses noise levels in the opposite direction, relabel $t_k$ accordingly---nothing below depends on the direction.
An Euler sampler with velocity field $v$ generates
\[
x_0^v=z,
\qquad
x_{k+1}^v=x_k^v+h\,v(x_k^v,t_k).
\]                                      
The reference field $v_\star$ generates the reference trajectory $x^\star$ in the same way.
We assume $v_\star(\cdot,t)$ is affine, so its one-step Jacobian
\[
A_k:=I_d+h\,\partial_x v_\star(\cdot,t_k)
\]
does not depend on the state.
Let
\[
\Pi_j:=A_{N-1}\cdots A_j,
\qquad
\Pi_N:=I_d.
\]
The terminal observable is affine:
\[
R(x)=c^\top x,
\qquad c\ne0.
\]
For a model field $v$, the terminal error is
\begin{equation}
\label{ins:eq:terminal-error}
\Psi(v):=R(x_N^v)-R(x_N^\star).
\end{equation}
For a parametrized model we write $\Psi(a):=\Psi(v(a))$.

\paragraph{Velocity error.}
Let $e_k$ denote the model's velocity error at step $k$.
In the main surrogate, $e_k$ is a per-step bias: it depends on the step index but not on the state.
The full residual is
\[
e=(e_0,\dots,e_{N-1}),
\qquad e_k\in\mathbb R^d.
\]
We equip residual sequences with the step-weighted inner product
\begin{equation}
\label{ins:eq:inner-product}
\langle a,b\rangle_h
:=h\sum_{k=0}^{N-1}a_k^\top b_k,
\qquad
\|a\|_h^2
=h\sum_{k=0}^{N-1}\|a_k\|_2^2.
\end{equation}
Up to the factor $T/2$, $\frac12\|e\|_h^2$ is the usual mean-squared velocity error.

\paragraph{Trainable residual corrections.}
We model the effect of training as adding trainable directions to a baseline residual.
Let $e^0$ be a baseline residual and let $U$ be a linear map whose columns are trainable correction directions.
Write
\begin{equation}
\label{ins:eq:affine-residual}
e(a)=e^0+Ua,
\qquad
\mathcal U:=\operatorname{range}(U),
\end{equation}
where $a$ denotes the coordinates of the trainable correction.
As $a$ ranges over all coordinates, $Ua$ ranges over all of $\mathcal U$, so the set of reachable residuals is exactly $e^0+\mathcal U$.
Both the step-local stage and the terminal-aligned stage move the same parameters, so they act through the same subspace $\mathcal U$.
Let $P_{\mathcal U}$ and $P_{\mathcal U^\perp}$ denote orthogonal projections onto $\mathcal U$ and $\mathcal U^\perp$ with respect to $\langle\cdot,\cdot\rangle_h$.

\paragraph{Terminal sensitivity.}
Define the terminal sensitivity vectors
\begin{equation}
\label{ins:eq:sensitivity}
s_k:=\Pi_{k+1}^\top c,
\qquad
s=(s_0,\dots,s_{N-1}),
\end{equation}
and split them into a reachable and an unreachable part,
\begin{equation}
\label{ins:eq:sensitivity-split}
s_{\mathcal U}:=P_{\mathcal U}s,
\qquad
s_{\mathcal U^\perp}:=P_{\mathcal U^\perp}s,
\qquad
s=s_{\mathcal U}+s_{\mathcal U^\perp}.
\end{equation}                            
The vector $s_k$ measures how strongly the terminal readout $R(x_N)$ reacts to a small velocity perturbation injected at step $k$.
It is determined by the sampler and the reference field; it does not depend on the model parameters.

\paragraph{Two objectives.}
The step-local loss is the quadratic
\begin{equation}
\label{ins:eq:step-local-loss}
\mathcal L(a):=\frac12\|e(a)\|_h^2.
\end{equation}
The terminal error, viewed as a functional of the residual, is the signed, sensitivity-weighted sum
\begin{equation}
\label{ins:eq:terminal-functional}
\Psi(a):=\langle e(a),s\rangle_h,
\end{equation}
which Lemma~\ref{ins:lem:pairing} shows to coincide with \eqref{ins:eq:terminal-error}.
The key contrast is:
\[
\mathcal L \text{ sees only } \|e_k\|_2^2,
\qquad
\Psi \text{ sees the signed sum } h\sum_k s_k^\top e_k.
\]

\paragraph{Useful expansion.}
For any perturbation $\delta a$,
\begin{equation}
\label{ins:eq:expansion}
\mathcal L(a+\delta a)
=
\mathcal L(a)
+
\langle e(a),U\delta a\rangle_h
+
\frac12\|U\delta a\|_h^2,
\end{equation}
and
\begin{equation}
\label{ins:eq:terminal-expansion}
\Psi(a+\delta a)-\Psi(a)
=
\langle U\delta a,s\rangle_h.
\end{equation}

\subsection{Step 1: Terminal error is the sensitivity-weighted sum of per-step errors}
\label{ins:sec:step1}

\begin{lemma}[Pairing identity]
\label{ins:lem:pairing}
\label{lem:bridge}                        
Under the affine surrogate, for every parameter $a$,
\begin{equation}
\label{ins:eq:pairing}
\Psi(a)
=
h\sum_{k=0}^{N-1}s_k^\top e_k(a)
=
\langle e(a),s\rangle_h .
\end{equation}

\end{lemma}

\begin{proof}
Let
\[
z_k:=x_k^{v(a)}-x_k^\star
\]
be the state deviation at step $k$, with $z_0=0$.
Subtracting the reference Euler update from the model Euler update gives
\[
z_{k+1}
=
z_k
+
h\bigl[v(a)(x_k^{v(a)},t_k)-v_\star(x_k^\star,t_k)\bigr].
\]
Because $v_\star$ is affine,
\[
v_\star(x_k^{v(a)},t_k)-v_\star(x_k^\star,t_k)
=
\partial_x v_\star(\cdot,t_k)z_k.
\]
By definition of the residual $e_k(a)$,
\[
v(a)(x_k^{v(a)},t_k)-v_\star(x_k^\star,t_k)
=
\partial_x v_\star(\cdot,t_k)z_k+e_k(a).
\]
Therefore
\[
z_{k+1}
=
A_k z_k+h e_k(a).
\]
Unrolling this recursion with $z_0=0$ yields
\[
z_N
=
h\sum_{k=0}^{N-1}\Pi_{k+1}e_k(a).
\]
Since $R(x)=c^\top x$,
\[
\Psi(a)
=
c^\top z_N
=
h\sum_{k=0}^{N-1}c^\top\Pi_{k+1}e_k(a)
=
h\sum_{k=0}^{N-1}(\Pi_{k+1}^\top c)^\top e_k(a)
=
h\sum_{k=0}^{N-1}s_k^\top e_k(a).
\]
This is exactly $\langle e(a),s\rangle_h$.
\end{proof}

\textbf{Intuition.}
The terminal error is not the sum of error magnitudes.
It is a signed sum: each per-step error $e_k$ is weighted by how sensitive the final output is to a perturbation at that step.
The step-local loss discards this sign and transport information; it only sees $\|e_k\|_2^2$.

\subsection{Step 2: Coherent errors accumulate; oscillatory errors cancel}
\label{ins:sec:step2}

By Lemma~\ref{ins:lem:pairing}, the terminal effect of a residual is the single number
\[
\Psi=\langle e,s\rangle_h.
\]
Whether the $N$ per-step contributions add or cancel depends on their alignment with $s$.

\begin{definition}[Coherent alignment]
\label{ins:def:coherence}
Let $\epsilon>0$ and $\mu\in(0,1]$.
We say that a residual $e$ is \emph{$(\mu,\epsilon)$-coherent} with respect to the terminal sensitivity $s$ if
\begin{equation}
\label{ins:eq:coherence}
\langle e,s\rangle_h
\ge
\mu\sum_{k=0}^{N-1}h\,\|s_k\|_2\,\|e_k\|_2,
\qquad
\|e_k\|_2\ge\epsilon
\quad\text{for every }k.
\end{equation}
\end{definition}

\noindent
The first inequality says that the per-step errors are not random jitter: they have a systematic component aligned with the direction that the terminal readout is sensitive to.
The second inequality says that the error is not concentrated on only a few steps.

\begin{proposition}[Accumulation versus cancellation]
\label{ins:prop:accumulation}
\label{prop:sparse_accumulation}
Let
\[
s_{\min}:=\min_k\|s_k\|_2,
\qquad
s_{\max}:=\max_k\|s_k\|_2,
\qquad
e_{\max}:=\max_k\|e_k\|_2.
\]
Then the terminal error $\Psi=\langle e,s\rangle_h$ satisfies the following.

\begin{enumerate}[label=(\alph*),leftmargin=1.9em,itemsep=0.25em]
\item \textbf{Upper bound.}
For every residual $e$,
\begin{equation}
\label{ins:eq:acc-upper}
|\Psi|
\le
\sum_{k=0}^{N-1}h\,\|s_k\|_2\,\|e_k\|_2
\le
s_{\max}\,e_{\max}\,T.
\end{equation}

\item \textbf{Coherent errors accumulate.}
If $e$ is $(\mu,\epsilon)$-coherent, then
\begin{equation}
\label{ins:eq:acc-lower}
\Psi
\ge
\mu\,s_{\min}\,\epsilon\,T.
\end{equation}

\item \textbf{Magnitude-only losses cannot distinguish accumulation from cancellation.}
There exist residuals with the same per-step magnitudes, hence the same step-local loss, but different terminal errors.
In the simplest constant-sensitivity case $s_k\equiv s_0$, alternating signs can make the terminal error vanish while keeping all $\|e_k\|_2$ fixed.
\end{enumerate}
\end{proposition}

\begin{proof}
(a) By the triangle inequality and Cauchy--Schwarz at each step,
\[
|\Psi|
=
\left|h\sum_{k=0}^{N-1}s_k^\top e_k\right|
\le
h\sum_{k=0}^{N-1}\|s_k\|_2\,\|e_k\|_2.
\]
Using $\|s_k\|_2\le s_{\max}$, $\|e_k\|_2\le e_{\max}$, and $hN=T$ gives
\[
|\Psi|
\le
h\sum_{k=0}^{N-1}s_{\max}e_{\max}
=
s_{\max}e_{\max}T.
\]

(b) By coherence,
\[
\Psi
=
\langle e,s\rangle_h
\ge
\mu\sum_{k=0}^{N-1}h\,\|s_k\|_2\,\|e_k\|_2.
\]
Since $\|s_k\|_2\ge s_{\min}$ and $\|e_k\|_2\ge\epsilon$,
\[
\Psi
\ge
\mu\sum_{k=0}^{N-1}h\,s_{\min}\epsilon
=
\mu\,s_{\min}\epsilon\,T.
\]

(c) Take the simplest case $d=1$, $s_k\equiv1$, and $N$ even.
Let $e_k=\epsilon$ for even $k$ and $e_k=-\epsilon$ for odd $k$.
Then every step has the same magnitude $\|e_k\|_2=\epsilon$, so the step-local loss is the same as for the constant residual $e_k\equiv\epsilon$.
But
\[
h\sum_{k=0}^{N-1}e_k=0,
\]
whereas the constant residual gives terminal error $T\epsilon$.
Thus a loss that depends only on $\|e_k\|_2^2$ cannot distinguish the two cases.
\end{proof}

\textbf{Intuition.}
The step-local loss sees only the size of each per-step error.
It does not see whether those errors point in the same direction and accumulate, or alternate and cancel.
For a highly sparse model, the concern is not merely that the residual error is nonzero, but that it may be \emph{coherent}: systematically aligned with the terminal sensitivity.
Such an error grows with the sampling horizon, while an oscillatory error of the same magnitude may cancel.

\subsection{Step 3: Step-local training removes the reachable part and stops}
\label{ins:sec:step3}

We now analyze what step-local training can and cannot remove.

Decompose the baseline residual into a trainable part and an orthogonal leftover:
\[
e^0
=
P_{\mathcal U}e^0
+
P_{\mathcal U^\perp}e^0.
\]
Define the leftover residual
\begin{equation}
\label{ins:eq:leftover}
\ell:=P_{\mathcal U^\perp}e^0,
\qquad
\beta:=\|\ell\|_h.
\end{equation}
If $\beta>0$, write $\hat\ell:=\ell/\beta$.

\begin{lemma}[Step-local optimum leaves the orthogonal leftover]
\label{ins:lem:step-local-optimum}
There exists a parameter $a_{\rm sp}$ such that
\[
Ua_{\rm sp}=-P_{\mathcal U}e^0.
\]
For every such $a_{\rm sp}$,
\[
e(a_{\rm sp})=\ell.
\]
Moreover, for any $a$, writing $g:=U(a-a_{\rm sp})\in\mathcal U$, one has
\begin{equation}
\label{ins:eq:loss-decomposition}
\mathcal L(a)
=
\frac12\|\ell\|_h^2
+
\frac12\|g\|_h^2.
\end{equation}
Hence every step-local minimizer has the same residual $e(a)=\ell$, and the minimum step-local loss is
\[
\min_a\mathcal L(a)=\frac12\|\ell\|_h^2.
\]
\end{lemma}

\begin{proof}
Since $P_{\mathcal U}e^0\in\mathcal U=\operatorname{range}(U)$, there exists $a_{\rm sp}$ with $Ua_{\rm sp}=-P_{\mathcal U}e^0$.
Then
\[
e(a_{\rm sp})=e^0+Ua_{\rm sp}
=
e^0-P_{\mathcal U}e^0
=
P_{\mathcal U^\perp}e^0
=
\ell.
\]
For any $a$, let $g=U(a-a_{\rm sp})$. Then
\[
e(a)=e(a_{\rm sp})+g=\ell+g,
\]
with $\ell\in\mathcal U^\perp$ and $g\in\mathcal U$.
By orthogonality,
\[
\|e(a)\|_h^2
=
\|\ell\|_h^2+\|g\|_h^2,
\]
which proves \eqref{ins:eq:loss-decomposition}.
Thus $\mathcal L(a)$ is minimized exactly when $g=0$, i.e.\ when $e(a)=\ell$.
\end{proof}

\textbf{Intuition.}
Step-local training removes exactly the part of the residual that lies in the trainable subspace $\mathcal U$.
It cannot remove the orthogonal leftover $\ell$.
At the step-local optimum, the loss gradient vanishes: every feasible parameter direction only increases the loss.
But this does not imply that the terminal error vanishes.

We now state the two conditions needed for the leftover to matter.

\begin{assumption}[Visibility and controllability]
\label{ins:as:visibility}
We assume:
\begin{enumerate}[label=(\Alph*),leftmargin=1.9em,itemsep=0.2em]
\item \textbf{Visibility.} The leftover is terminally visible:
\[
\beta>0,
\qquad
\langle \hat\ell,s\rangle_h\ge c_0>0.
\]
If $\langle \hat\ell,s\rangle_h<0$, replace $R$ by $-R$ once; the case $\langle \hat\ell,s\rangle_h=0$ is excluded.

\item \textbf{Controllability.} The terminal sensitivity has a reachable component:
$s_{\mathcal U}=P_{\mathcal U}s\ne0$.
\end{enumerate}
\end{assumption}

\noindent
Visibility says that the leftover contributes to the terminal readout rather than canceling inside it.
Controllability says that the trainable directions can influence the terminal readout at all.
If $s_{\mathcal U}=0$, then no parameter update in $\mathcal U$ can change the terminal functional.

\begin{proposition}[The terminal error survives the step-local optimum]
\label{ins:prop:survive}
Under Assumption~\ref{ins:as:visibility}, at any step-local minimizer $a_{\rm sp}$,
\begin{equation}
\label{ins:eq:surviving-terminal-error}
\Psi_{\rm sp}
:=
\Psi(a_{\rm sp})
=
\langle \ell,s\rangle_h
=
\langle \ell,s_{\mathcal U^\perp}\rangle_h
=
\beta\,\langle \hat\ell,s\rangle_h
\ge
c_0\beta
>
0.
\end{equation}
Moreover:
\begin{itemize}[leftmargin=1.4em,itemsep=0.2em]
\item the step-local loss has zero first-order variation at $a_{\rm sp}$;
\item the terminal functional has nonzero first-order variation at $a_{\rm sp}$ through the same trainable subspace $\mathcal U$.
\end{itemize}
Equivalently, the step-local gradient vanishes, but the gradient of $\frac12\Psi^2$ need not vanish.
\end{proposition}

\begin{proof}
Since $e(a_{\rm sp})=\ell$, Lemma~\ref{ins:lem:pairing} gives
$\Psi_{\rm sp}=\langle \ell,s\rangle_h$.
Because $\ell\in\mathcal U^\perp$ and $s_{\mathcal U}\in\mathcal U$, the split \eqref{ins:eq:sensitivity-split} gives
$\langle \ell,s_{\mathcal U}\rangle_h=0$, hence
$\Psi_{\rm sp}=\langle \ell,s_{\mathcal U^\perp}\rangle_h$.
Using $\ell=\beta\hat\ell$ gives
$\Psi_{\rm sp}=\beta\langle \hat\ell,s\rangle_h\ge c_0\beta>0$.

For the loss, take any coordinate perturbation $\delta a$ and let $g=U\delta a\in\mathcal U$.
By \eqref{ins:eq:expansion}, the first-order change of $\mathcal L$ at $a_{\rm sp}$ is
$\langle \ell,g\rangle_h=0$, because $\ell\perp\mathcal U$.
Thus the step-local loss is stationary.

For the terminal functional, by \eqref{ins:eq:terminal-expansion},
\[
\Psi(a_{\rm sp}+\delta a)-\Psi(a_{\rm sp})
=
\langle U\delta a,s\rangle_h.
\]

Since $s_{\mathcal U}\ne0$ and $s_{\mathcal U}\in\mathcal U$, there exists a perturbation $\delta a$ with $U\delta a=s_{\mathcal U}$.
For this perturbation,
\[
\langle U\delta a,s\rangle_h
=
\langle s_{\mathcal U},s\rangle_h
=
\|s_{\mathcal U}\|_h^2
>
0.
\]
Hence the terminal functional has a nonzero first-order direction, while the step-local loss does not.
\end{proof}

\textbf{Intuition.}
At the step-local optimum, the model has already removed every error component that the trainable subspace can express.
The remaining component $\ell$ is invisible to further step-local training: the loss gradient is zero.
But if $\ell$ is aligned with the terminal sensitivity, it still produces a nonzero terminal error.
At the same time, the terminal functional can still be changed by moving in the same trainable subspace.
This mismatch---zero step-local gradient but nonzero terminal gradient---is the core of the high-sparsity trap in the surrogate.

\subsection{Step 4: Terminal-aligned correction and its surrogate price}
\label{ins:sec:step4}

We now ask: if the step-local optimum leaves a terminal-visible residual, how much step-local sub-optimality must be paid to cancel it?

Let $a_{\rm sp}$ be a step-local minimizer and define the step-local sub-optimality
\begin{equation}
\label{ins:eq:suboptimality}
\Delta(a)
:=
\mathcal L(a)-\mathcal L(a_{\rm sp}).
\end{equation}
By Lemma~\ref{ins:lem:step-local-optimum}, if $g:=U(a-a_{\rm sp})$, then
\begin{equation}
\label{ins:eq:delta-as-norm}
\Delta(a)=\frac12\|g\|_h^2,
\end{equation}
and as $a$ ranges over all parameters, $g$ ranges over all of $\mathcal U$.
Also, since $g\in\mathcal U$,
\begin{equation}
\label{ins:eq:terminal-shift}
\Psi(a)
=
\Psi_{\rm sp}
+
\langle g,s\rangle_h
=
\Psi_{\rm sp}
+
\langle g,s_{\mathcal U}\rangle_h.
\end{equation}

\begin{corollary}[Surrogate price of terminal cancellation]
\label{ins:cor:price}
\label{cor:price}
Under Assumption~\ref{ins:as:visibility}, the minimum step-local sub-optimality required to make the terminal functional vanish is
\begin{equation}
\label{ins:eq:price}
\mathcal C
:=
\min\{\Delta(a):\Psi(a)=0\}
=
\frac{\Psi_{\rm sp}^2}{2\,\|s_{\mathcal U}\|_h^2}
>
0.
\end{equation}
The minimum is attained exactly at those $a$ with $U(a-a_{\rm sp})=g^\star$, where
\begin{equation}
\label{ins:eq:optimal-direction}
g^\star
=
-\frac{\Psi_{\rm sp}}{\|s_{\mathcal U}\|_h^2}\,s_{\mathcal U}.
\end{equation}
Moreover, for every $a$,
\begin{equation}
\label{ins:eq:terminal-bound}
|\Psi(a)-\Psi_{\rm sp}|
\le
\|s_{\mathcal U}\|_h\sqrt{2\Delta(a)}.
\end{equation}
\end{corollary}

\begin{proof}
By \eqref{ins:eq:terminal-shift}, the constraint $\Psi(a)=0$ is
\[
\langle g,s_{\mathcal U}\rangle_h=-\Psi_{\rm sp}.
\]
By Cauchy--Schwarz,
\[
|\Psi_{\rm sp}|
=
|\langle g,s_{\mathcal U}\rangle_h|
\le
\|g\|_h\,\|s_{\mathcal U}\|_h,
\]
so any feasible $g$ satisfies
$\|g\|_h^2\ge\Psi_{\rm sp}^2/\|s_{\mathcal U}\|_h^2$.
Using $\Delta=\frac12\|g\|_h^2$ gives
\[
\Delta
\ge
\frac{\Psi_{\rm sp}^2}{2\,\|s_{\mathcal U}\|_h^2}.
\]
Equality holds if and only if $g$ is parallel to $s_{\mathcal U}$ and satisfies the constraint, i.e.\ $g=g^\star$; such $g$ is feasible because $s_{\mathcal U}\in\mathcal U$.
The corresponding step-local sub-optimality is
\[
\mathcal C
=
\frac12\|g^\star\|_h^2
=
\frac{\Psi_{\rm sp}^2}{2\,\|s_{\mathcal U}\|_h^2}.
\]
Finally, for any $a$, using \eqref{ins:eq:terminal-shift} and Cauchy--Schwarz,
\[
|\Psi(a)-\Psi_{\rm sp}|
=
|\langle g,s_{\mathcal U}\rangle_h|
\le
\|g\|_h\,\|s_{\mathcal U}\|_h
=
\|s_{\mathcal U}\|_h\sqrt{2\Delta(a)},
\]
which proves \eqref{ins:eq:terminal-bound}.
\end{proof}

\textbf{Intuition.}
The terminal constraint $\Psi(a)=0$ is a single linear constraint on the reachable residual correction $g\in\mathcal U$.
Among all corrections satisfying this constraint, the cheapest one in step-local loss is the one aligned with the reachable terminal sensitivity $s_{\mathcal U}$.
Its cost is exactly the squared ratio of the surviving terminal error to the amount of terminal sensitivity that training can reach.

\begin{proposition}[Terminal descent and the trade-off frontier]
\label{ins:prop:terminal_descent}
\label{prop:dm_descent}
Under Assumption~\ref{ins:as:visibility}, fix a step-local minimizer $a_{\rm sp}$ and let $\mathcal C$ and $g^\star$ be as in Corollary~\ref{ins:cor:price}.

\begin{enumerate}[label=(\arabic*),leftmargin=1.9em,itemsep=0.3em]
\item \textbf{Terminal descent exists where the step-local loss is flat.}
At $a_{\rm sp}$, the step-local loss has zero first-order variation.
However, the terminal squared error
$J(a):=\frac12\Psi(a)^2$
has a nonzero first-order descent direction through the same trainable subspace $\mathcal U$.

\item \textbf{Trade-off frontier.}
For every step-local budget $B\ge0$,
\begin{equation}
\label{ins:eq:frontier}
\min\bigl\{|\Psi(a)|:\Delta(a)\le B\bigr\}
=
\Bigl(
\Psi_{\rm sp}
-
\|s_{\mathcal U}\|_h\sqrt{2B}
\Bigr)_+.
\end{equation}
In particular, the terminal error decreases at first order in the parameter displacement, while the step-local loss increases only at second order.

\item \textbf{Optimal path.}
For $t\in[0,1]$, let $a(t)$ be any parameter satisfying
$U(a(t)-a_{\rm sp})=t\,g^\star$.
Then
\begin{equation}
\label{ins:eq:optimal-path}
\Psi(a(t))=(1-t)\Psi_{\rm sp},
\qquad
\Delta(a(t))=t^2\mathcal C,
\qquad
\|e(a(t))\|_h^2=\beta^2+2t^2\mathcal C.
\end{equation}
Thus the terminal error decreases linearly along this path, while the step-local cost increases quadratically.
\end{enumerate}
\end{proposition}

\begin{proof}
(1)
By Lemma~\ref{ins:lem:step-local-optimum}, the first-order variation of $\mathcal L$ at $a_{\rm sp}$ vanishes because the residual $\ell$ is orthogonal to every reachable correction $g\in\mathcal U$.
For the terminal functional, choose a perturbation $\delta a$ such that $U\delta a=-s_{\mathcal U}$; such a perturbation exists because $s_{\mathcal U}\in\mathcal U$.
Then, using \eqref{ins:eq:terminal-shift},
\[
\frac{d}{dt}\Big|_{t=0}\Psi(a_{\rm sp}+t\,\delta a)
=
\langle -s_{\mathcal U},s\rangle_h
=
-\|s_{\mathcal U}\|_h^2
<0.
\]
Since $\Psi_{\rm sp}>0$, the derivative of $J=\frac12\Psi^2$ along $\delta a$ is
$\Psi_{\rm sp}\cdot(-\|s_{\mathcal U}\|_h^2)<0$.
Hence $J$ has a first-order descent direction, while $\mathcal L$ does not.

(2)
For any $a$ with $\Delta(a)\le B$, write $g=U(a-a_{\rm sp})$.
Then $\frac12\|g\|_h^2\le B$, so $\|g\|_h\le\sqrt{2B}$.
By \eqref{ins:eq:terminal-shift} and Cauchy--Schwarz,
\[
|\Psi(a)-\Psi_{\rm sp}|
=
|\langle g,s_{\mathcal U}\rangle_h|
\le
\|g\|_h\,\|s_{\mathcal U}\|_h
\le
\|s_{\mathcal U}\|_h\sqrt{2B},
\]
so $|\Psi(a)|\ge\Psi_{\rm sp}-\|s_{\mathcal U}\|_h\sqrt{2B}$, and since $|\Psi(a)|\ge0$,
\[
|\Psi(a)|
\ge
\Bigl(
\Psi_{\rm sp}
-
\|s_{\mathcal U}\|_h\sqrt{2B}
\Bigr)_+.
\]
For achievability, set
\[
g_B
=
-\min\left\{
\sqrt{2B},\
\frac{\Psi_{\rm sp}}{\|s_{\mathcal U}\|_h}
\right\}
\frac{s_{\mathcal U}}{\|s_{\mathcal U}\|_h}
\ \in\mathcal U,
\]
and let $a_B$ be any parameter with $U(a_B-a_{\rm sp})=g_B$.   
Then $\Delta(a_B)=\frac12\|g_B\|_h^2\le B$ and, by \eqref{ins:eq:terminal-shift},
\[
\Psi(a_B)
=
\Psi_{\rm sp}
-
\min\bigl\{
\|s_{\mathcal U}\|_h\sqrt{2B},\
\Psi_{\rm sp}
\bigr\},
\]
which attains \eqref{ins:eq:frontier}.

(3)
Let $g=t g^\star$.
From \eqref{ins:eq:terminal-shift} and \eqref{ins:eq:optimal-direction},
\[
\Psi(a(t))
=
\Psi_{\rm sp}
+
t\langle g^\star,s_{\mathcal U}\rangle_h
=
\Psi_{\rm sp}
-
t\Psi_{\rm sp}
=
(1-t)\Psi_{\rm sp}.
\]
Also $\Delta(a(t))=\frac12\|t g^\star\|_h^2=t^2\mathcal C$.
Finally, since $e(a(t))=\ell+t g^\star$ with $\ell\in\mathcal U^\perp$ and $g^\star\in\mathcal U$,
\[
\|e(a(t))\|_h^2
=
\|\ell\|_h^2+t^2\|g^\star\|_h^2
=
\beta^2+2t^2\mathcal C.
\]
\end{proof}

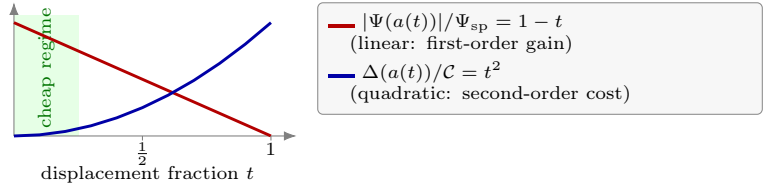
\begin{figure}[t]
\centering
\begin{tikzpicture}[x=3.4cm,y=1.5cm,>=Latex,font=\small,line join=round]
\fill[green!10] (0,0) rectangle (0.25,1.06);
\node[green!45!black,font=\scriptsize,rotate=90,anchor=center] at (0.125,0.60)
      {cheap regime};
\draw[black!50,->] (0,0)--(1.10,0);
\draw[black!50,->] (0,0)--(0,1.18);
\draw[red!70!black,line width=1.1pt] (0,1)--(1,0);
\draw[blue!65!black,line width=1.1pt]
 (0,0) -- (0.1,0.01) -- (0.2,0.04) -- (0.3,0.09) -- (0.4,0.16) -- (0.5,0.25)
 -- (0.6,0.36) -- (0.7,0.49) -- (0.8,0.64) -- (0.9,0.81) -- (1.0,1.00);
\draw[black!50] (0.5,0)--(0.5,-0.04);
\node[font=\scriptsize,anchor=north,inner sep=1pt] at (0.5,-0.03) {$\tfrac12$};
\draw[black!50] (1,0)--(1,-0.04);
\node[font=\scriptsize,anchor=north,inner sep=1pt] at (1,-0.03) {$1$};
\node[font=\scriptsize,anchor=north,inner sep=1pt] at (0.52,-0.24)
      {displacement fraction $t$};
\node[draw=black!45,fill=black!3,rounded corners=2pt,inner sep=4pt,
      font=\scriptsize,align=left,text width=5.6cm,anchor=north west] at (1.18,1.18)
 {\textcolor{red!70!black}{\rule{10pt}{1.2pt}}~$|\Psi(a(t))|/\Psi_{\rm sp}=1-t$\\
  \hspace*{1.0em}(linear: first-order gain)\\[3pt]
  \textcolor{blue!65!black}{\rule{10pt}{1.2pt}}~$\Delta(a(t))/\mathcal C=t^{2}$\\
  \hspace*{1.0em}(quadratic: second-order cost)};
\end{tikzpicture}
\caption{
\textbf{The exchange rate inside the surrogate}
(Corollary~\ref{ins:cor:price}, Proposition~\ref{ins:prop:terminal_descent}).
Along the optimal path the terminal error falls \emph{linearly} in the displacement fraction $t$
while the step-local loss rises only \emph{quadratically}; the residual norm
$\|e(a(t))\|_h^2=\beta^2+2t^2\mathcal C$ follows the same quadratic.
This is the formal reason for a \emph{staged} recipe rather than a single joint objective.
Both axes are surrogate quantities: neither is FID, sliced-$W_2$, or wall-clock training cost.
}
\label{fig:ins-frontier}
\end{figure}

\textbf{Intuition.}
At the step-local optimum, the step-local loss is locally flat: moving a small amount costs only quadratically in $\mathcal L$.
But the terminal error changes linearly with the same move.
Therefore a terminal-aligned signal can obtain a first-order reduction in terminal error while paying only a second-order step-local price.
This is the formal reason for the staged recipe:
first fit the step-local objective, then spend a small amount of step-local sub-optimality on terminal alignment.

\subsection{Step 5: Cancellation is not removal}
\label{ins:sec:step5}

The previous step shows that terminal alignment can reduce the terminal error.
However, it does not repair the underlying velocity error.

\begin{corollary}[Residual inflation]
\label{ins:cor:inflation}
Under Assumption~\ref{ins:as:visibility}, let $a$ satisfy $\Psi(a)=0$.
Writing $g=U(a-a_{\rm sp})$, one has
\begin{equation}
\label{ins:eq:inflation}
\|e(a)\|_h^2
=
\beta^2+\|g\|_h^2
\ge
\beta^2+2\mathcal C,
\end{equation}
with equality if and only if $g=g^\star$.
Along the optimal path $g=t g^\star$, the terminal error decreases monotonically,
$|\Psi(a(t))|=(1-t)\Psi_{\rm sp}$,
while the residual norm increases monotonically,
$\|e(a(t))\|_h^2=\beta^2+2t^2\mathcal C$.
\end{corollary}

\begin{proof}
Since $e(a)=\ell+g$ with $\ell\in\mathcal U^\perp$ and $g\in\mathcal U$,
\[
\|e(a)\|_h^2
=
\|\ell\|_h^2+\|g\|_h^2
=
\beta^2+\|g\|_h^2.
\]
Among all $g\in\mathcal U$ with $\Psi(a)=0$, Corollary~\ref{ins:cor:price} shows that the minimum possible $\frac12\|g\|_h^2$ is $\mathcal C$, attained uniquely by $g=g^\star$.
Thus $\|e(a)\|_h^2\ge\beta^2+2\mathcal C$.
The statements along the optimal path follow directly from \eqref{ins:eq:optimal-path}.
\end{proof}

\textbf{Intuition.}
Terminal alignment does not remove the velocity error.
It adds a compensating component whose terminal effect cancels the terminal-visible projection of the leftover error.
As a result, the per-step residual becomes larger, even though the terminal readout becomes smaller.
The correct wording is therefore:
\emph{terminal alignment cancels the terminal-visible projection of the error}, not
\emph{terminal alignment removes the error}.

\subsection{Scope and limitations}
\label{ins:sec:scope}

\begin{remark}[What is and is not established]
\label{ins:rem:scope}
\label{rem:insight_scope}
The preceding argument is a theorem about the surrogate model.
It establishes three mechanism-level statements.

\begin{itemize}[leftmargin=1.4em,itemsep=0.25em]
\item[(a)] Coherent per-step velocity errors can accumulate into a terminal error, while oscillatory errors of the same magnitude can cancel. A magnitude-based step-local loss cannot distinguish these cases.

\item[(b)] At a step-local optimum, the remaining terminal error can be strictly positive even though the step-local gradient vanishes. If the leftover is terminally visible and the terminal sensitivity has a reachable component, then a terminal-aligned signal can reduce this error through the same trainable subspace.

\item[(c)] Cancelling the terminal-visible projection of the error costs a computable surrogate step-local price and enlarges the residual norm. It is cancellation, not repair.
\end{itemize}

The surrogate is deliberately idealized.
It is not a theorem about nonlinear sparse video diffusion transformers, optimizer trajectories, or concrete distillation losses.
The terminal quantity $\Psi$ is a single scalar functional of the terminal state; it is not FID, sliced-$W_2$, VBench, or any perceptual metric.
No specific distribution-matching or consistency loss is verified to descend $\Psi$.
Therefore the main text should read this appendix as motivation for a terminal-aligned stage, not as a guarantee about any particular generation metric.
\end{remark}

\subsection{A closed-form illustration}
\label{subsec:insight_toy}

We give a small exact example to make the geometry concrete.
Let
\[
d=1,\qquad N=3,\qquad h=\frac13,\qquad T=1,
\]
and take
\[
v_\star(x,t)=3x,
\qquad
R(x)=x.
\]
Then each one-step Jacobian is $A_k=1+3h=2$, the terminal sensitivities are
\[
s=(4,2,1),
\]
and with the weighted inner product $\langle a,b\rangle_h=\frac13\sum_k a_k b_k$,
\[
\|s\|_h^2
=
\frac13(16+4+1)
=
7.
\]

\paragraph{One trainable direction.}
Let the trainable subspace be spanned by $\phi=(1,1,-2)$, so that
\[
\|\phi\|_h^2
=
\frac13(1+1+4)
=
2.
\]
Let the baseline residual be
\[
e^0=\beta\hat\ell+\gamma\phi,
\qquad
\hat\ell=(1,1,1),
\qquad \beta>0 .
\]                                        
Since $\|\hat\ell\|_h=1$ and $\langle\hat\ell,\phi\rangle_h=0$, the step-local leftover is
$\ell=\beta\hat\ell$.
The surviving terminal error is
\[
\Psi_{\rm sp}
=
\langle \ell,s\rangle_h
=
\frac{\beta}{3}(4+2+1)
=
\frac73\beta,
\]
so the visibility constant is $c_0=\langle\hat\ell,s\rangle_h=\frac73$.
The reachable part of the terminal sensitivity is
\[
s_{\mathcal U}
=
\frac{\langle s,\phi\rangle_h}{\|\phi\|_h^2}\phi
=
\frac{4/3}{2}\phi
=
\frac23\phi,
\qquad
\|s_{\mathcal U}\|_h^2
=
\frac49\cdot2
=
\frac89.
\]
The surrogate price is therefore
\[
\mathcal C
=
\frac{\Psi_{\rm sp}^2}{2\|s_{\mathcal U}\|_h^2}
=
\frac{(7/3)^2\beta^2}{2\cdot 8/9}
=
\frac{49}{16}\beta^2,
\]
and the optimal correction is
\[
g^\star
=
-\frac{\Psi_{\rm sp}}{\|s_{\mathcal U}\|_h^2}s_{\mathcal U}
=
-\frac{(7/3)\beta}{8/9}\cdot\frac23\phi
=
-\frac74\beta\,\phi .
\]                                        
Thus the corrected residual is
\[
e^\star
=
\beta(1,1,1)-\frac74\beta(1,1,-2)
=
\beta\left(-\frac34,-\frac34,\frac92\right),
\]
and one checks directly that
\[
\langle e^\star,s\rangle_h
=
\frac{\beta}{3}\left(4\cdot\Bigl(-\frac34\Bigr)+2\cdot\Bigl(-\frac34\Bigr)+1\cdot\frac92\right)
=0,
\qquad
\|e^\star\|_h^2
=
\beta^2+2\mathcal C
=
\frac{57}{8}\beta^2 .
\]                                        
The terminal error has been cancelled, but the residual norm has increased.
Along the path of Proposition~\ref{ins:prop:terminal_descent}(3), at $t=\frac12$ the terminal error
is halved for one quarter of the price, $\Delta=\frac{49}{64}\beta^2$.

\paragraph{Two trainable directions.}
Now add a second trainable direction $\psi=(1,-1,0)$, so that
\[
\mathcal U=\operatorname{span}\{\phi,\psi\},
\qquad
\mathcal U^\perp=\operatorname{span}\{(1,1,1)\},
\]
and the leftover is still $\ell=\beta(1,1,1)$, hence $\Psi_{\rm sp}=\frac73\beta$ as before.
The reachable part of $s$ is now
\[
s_{\mathcal U}
=
s-\langle s,\hat\ell\rangle_h\,\hat\ell
=
\left(\frac53,-\frac13,-\frac43\right),
\qquad
\|s_{\mathcal U}\|_h^2
=
\frac{14}{9},
\]
so the price becomes
\[
\mathcal C
=
\frac{(7/3)^2\beta^2}{2\cdot 14/9}
=
\frac{7}{4}\beta^2,
\]
and the optimal correction is
\[
g^\star
=
-\frac{\Psi_{\rm sp}}{\|s_{\mathcal U}\|_h^2}s_{\mathcal U}
=
-\frac{(7/3)\beta}{14/9}\,s_{\mathcal U}
=
-\frac32\beta\,s_{\mathcal U}
=
\beta\left(-\frac52,\frac12,2\right).
\]                                        
The corrected residual is therefore
\[
e^\star
=
\beta(1,1,1)+g^\star
=
\beta\left(-\frac32,\frac32,3\right),
\qquad
\|e^\star\|_h^2
=
\beta^2+2\mathcal C
=
\frac92\beta^2 .
\]

In this two-dimensional case, the constraint $\Psi(a)=0$ is still one linear equation, but the trainable subspace has dimension two.
Therefore there are infinitely many corrections $g\in\mathcal U$ that cancel the terminal error.
Among them, $g^\star$ is the unique correction with minimum step-local cost.
Any additional component in $\mathcal U$ orthogonal to $s_{\mathcal U}$ leaves the terminal error unchanged but increases the residual norm and the step-local sub-optimality.

\textbf{Summary of the illustration.}
This example shows the three core mechanisms in exact arithmetic:
\begin{itemize}[leftmargin=1.4em,itemsep=0.2em]
\item the step-local optimum leaves a nonzero terminal error $\Psi_{\rm sp}$;
\item a terminal-aligned correction through the same trainable subspace can cancel $\Psi_{\rm sp}$;
\item the cancellation costs a finite surrogate step-local price $\mathcal C$ and enlarges the residual norm.
\end{itemize}
Thus the appendix gives a closed-form picture of why a terminal-aligned stage can help after step-local sparse training, while also making clear that it cancels the terminal-visible projection of the error rather than removing the underlying velocity error.

\section{Generality across Trainable Sparse Designs}
\label{app:generality}

We verify that the high-sparsity trap is not specific to the RoLA selector, and that strong baseline recipes do not survive $97\%$ sparsity.
All experiments use the Wan2.1-T2V-14B-480P testbed of Sec.~\ref{sec:trap}: the same dense checkpoint, the same enlarged training set, and the same evaluation protocol (terminal error and VBench at $480\times832$).

\textbf{Step-local training across selectors.}
We instantiate two representative trainable sparse branches---VSA-style block selection~\citep{zhang2026faster} and SLA-style sparse--linear attention~\citep{zhang2026sla}---and train them with the step-local flow-matching loss~\eqref{eq:fm-loss} at $90\%$/$95\%$/$97\%$ sparsity under the extended budget of Sec.~\ref{sec:trap} (10{,}000 steps).

\textbf{Baseline recipe at $97\%$.}
We retrain the full FastWan (VSA) recipe~\citep{zhang2026faster}---joint sparse-attention training with DMD-style distillation---at $97\%$ sparsity from the same dense checkpoint, changing only the sparsity level in its official training configuration.

\textbf{Staged recipe across selectors.}
We apply the Stage-2 trajectory-mixed distillation of Sec.~\ref{subsec:distill} to the step-local VSA- and SLA-style $97\%$ checkpoints without modifying their sparse branches.

Table~\ref{tab:generality} shows the same trap signatures as Sec.~\ref{sec:trap} on both selectors: the step-local validation loss plateaus at a low value at every sparsity level, while the terminal error and VBench degrade sharply at $95\%$--$97\%$.
SLA attains a consistently lower validation loss than VSA at matched sparsity, yet its terminal error is only marginally smaller---a lower step-local loss does not translate into better terminal quality.
Retraining the full FastWan recipe at $97\%$ degrades sharply relative to its official $90\%$ operating point, while \method at the same $97\%$ sparsity remains close to the dense model. 
The terminal-aligned Stage-2 stage restores terminal quality on both selectors, confirming that the trap is induced by step-local supervision rather than by any specific sparse design, and that the staged recipe is selector-agnostic.

\begin{table}[htbp]
\centering
\small
\setlength{\tabcolsep}{6pt}
\renewcommand{\arraystretch}{1.15}
\caption{
The high-sparsity trap across trainable sparse designs (Wan2.1-T2V-14B-480P).
Step-local-only rows are trained from the same dense checkpoint with the flow-matching loss~\eqref{eq:fm-loss} under the extended 10{,}000-step budget of Sec.~\ref{sec:trap}; terminal error is the paired latent MSE of Sec.~\ref{sec:trap}.
``FastWan recipe'' retrains the official VSA training pipeline~\citep{zhang2026faster} with only the sparsity level changed to $97\%$, and is evaluated under its official inference settings; ``$+$\,Stage~2'' applies our trajectory-mixed distillation to the corresponding $97\%$ checkpoint.
}
\label{tab:generality}
\arrayrulecolor{SparkDiffusionRed}
\begin{tabular}{llccc}
\toprule[1.25pt]
\rowcolor{SparkDiffusionHeader}
\textbf{Sparse branch / recipe} & \textbf{Sp.} & \textbf{Val loss} $\downarrow$ & \textbf{Term.\ error} $\downarrow$ & \textbf{VBench} $\uparrow$ \\
\midrule[0.65pt]
VSA-style, step-local & $90\%$ & 0.0892 & 0.1205 & 81.75 \\ 
VSA-style, step-local & $95\%$ & 0.0947 & 0.1232 & 80.89 \\ 
VSA-style, step-local & $97\%$ & 0.0998 & 0.1255 & 79.72 \\ 
SLA-style, step-local & $90\%$ & 0.0871 & 0.1203 & 81.87 \\ 
SLA-style, step-local & $95\%$ & 0.0932 & 0.1227 & 80.94 \\ 
SLA-style, step-local & $97\%$ & 0.0979 & 0.1251 & 79.84 \\ 
\midrule[0.4pt]
FastWan (VSA) recipe, retrained & $97\%$ & 0.0996 & 0.1198 & 80.23 \\ 
\rowcolor{SparkDiffusionLight}
VSA-style $+$ Stage 2 & $97\%$ & 0.0995 & 0.1001 & 82.45 \\ 
\rowcolor{SparkDiffusionLight}
SLA-style $+$ Stage 2 & $97\%$ & 0.0977 & 0.0987 & 82.51 \\ 
\rowcolor{SparkDiffusionLight}
RoLA $+$ Stage 2 (\method) & $97\%$ & 0.0953 & 0.0921 & 82.99 \\ 
\bottomrule[1.25pt]
\end{tabular}
\arrayrulecolor{black}
\end{table}
\section{Training Details}
\label{app:training_details}

This section reports the training configurations of \method for the primary backbone Wan2.1-T2V-14B (stages defined in Sec.\ref{sec:method}; Stage~2 initializes from the Stage~1 checkpoint), together with hardware, wall-clock time, and GPU hours.

\subsection{Stage 1: Sparse Warm-up}
\label{app:stage1_training}

Stage~1 starts from the pretrained dense Wan2.1-T2V-14B checkpoint and
replaces dense self-attention with the compensated sparse attention
of Sec.~\ref{subsec:sparse}, optimized with the native flow-matching
objective alone (no auxiliary or layerwise losses).

Table~\ref{tab:stage1_training_config} summarizes the data configuration,
sparse-attention architecture, optimization hyperparameters, and training
cost used in this stage.

\begin{table}[H]
\centering
\small
\setlength{\tabcolsep}{7pt}
\renewcommand{\arraystretch}{1.16}

\caption{
Training configuration of Stage~1 sparse warm-up on Wan2.1-T2V-14B.
}
\label{tab:stage1_training_config}

\arrayrulecolor{SparkDiffusionRed}

\begin{tabular*}{\linewidth}{
@{\extracolsep{\fill}}
p{0.34\linewidth}
p{0.58\linewidth}
@{}
}

\toprule[1.25pt]

\rowcolor{SparkDiffusionHeader}
\textbf{Hyperparameter}
&
\textbf{Setting}
\\

\midrule[0.65pt]

\rowcolor{SparkDiffusionLight}
\multicolumn{2}{l}{
\textcolor{SparkDiffusionRed}{\textbf{Model \& Data}}
}
\\

Backbone
& Wan2.1-T2V-14B \\

Training dataset
& OpenVid subset~\citep{nan2025openvid} \\

Number of training videos
& 2{,}000 \\

Video setting
& $480\times832$ (480P), $720\times1280$ (720P), 81 frames \\

\addlinespace[1pt]
\rowcolor{SparkDiffusionLight}
\multicolumn{2}{l}{
\textcolor{SparkDiffusionRed}{\textbf{Sparse Architecture}}
}
\\

Attention sparsity
& 97\% \\

Sparse block size
& $64\times64$ \\

Compensation rank
& 64 \\

Trainable modules
& Full DiT backbone (including our low-rank compensation modules and gating parameters) \\

Frozen modules
& VAE and text encoder \\

\addlinespace[1pt]
\rowcolor{SparkDiffusionLight}
\multicolumn{2}{l}{
\textcolor{SparkDiffusionRed}{\textbf{Optimization}}
}
\\

Training objective
& Flow matching \\

Optimizer
& AdamW \\

Learning rate
& $2\times10^{-6}$ (backbone);
$5\times10^{-5}$ (low-rank projections / gate bias);
$3\times10^{-5}$ (gate projection) \\

Weight decay
& $1\times10^{-5}$ (backbone); $0$ (others) \\

LR schedule
& Cosine decay \\

Training epochs
& 4 ($\approx$250 steps) \\

Global batch size
& 32 \\

Gradient clipping
& 1.0 \\

Training precision
& BF16 \\

Gradient checkpointing
& Block-wise activation checkpointing \\

Random seed
& 0 \\

\addlinespace[1pt]
\rowcolor{SparkDiffusionLight}
\multicolumn{2}{l}{
\textcolor{SparkDiffusionRed}{\textbf{Compute}}
}
\\

Hardware
& $8\times$ NVIDIA H100 80GB SXM \\

Parallelism
& FSDP full-shard \\

Wall-clock training time
& $\approx$2.9 hours \\

GPU Hours
& $\approx$23 GPU hours \\

\bottomrule[1.25pt]

\end{tabular*}

\arrayrulecolor{black}
\end{table}

\subsection{Stage 2: Trajectory-Mixed Distillation}
\label{app:stage2_training}

Stage~2 initializes the student from the Stage~1 checkpoint and freezes the dense Wan2.1-T2V-14B model as the teacher, following the CrossDistill schedule of Sec.~\ref{subsec:distill}: high-noise PCM-style consistency, low-noise DMD-style distribution matching, and alternating student/critic updates.

Table~\ref{tab:stage2_training_config} reports the complete distillation
schedule, distribution-matching configuration, optimization
hyperparameters, and computational cost.

\begin{table}[H]
\centering
\small
\setlength{\tabcolsep}{7pt}
\renewcommand{\arraystretch}{1.14}

\caption{
Training configuration of Stage~2 trajectory-mixed distillation on
Wan2.1-T2V-14B.
}
\label{tab:stage2_training_config}

\arrayrulecolor{SparkDiffusionRed}

\begin{tabular*}{\linewidth}{
@{\extracolsep{\fill}}
p{0.34\linewidth}
p{0.58\linewidth}
@{}
}

\toprule[1.25pt]

\rowcolor{SparkDiffusionHeader}
\textbf{Hyperparameter}
&
\textbf{Setting}
\\

\midrule[0.65pt]

\rowcolor{SparkDiffusionLight}
\multicolumn{2}{l}{
\textcolor{SparkDiffusionRed}{\textbf{Model setting}}
}
\\

Student backbone
& Wan2.1-T2V-14B \\

Student initialization
& Stage~1 sparse warm-up checkpoint \\

Teacher model
& Frozen dense Wan2.1-T2V-14B \\

Teacher CFG scale
& 5.0  \\

Training dataset
& Teacher-synthesized T2V dataset~\cite{zheng2026large} \\


Video setting
& $480\times832$ (480P), $720\times1280$ (720P), 81 frames \\

Attention sparsity
& 97\% \\

\addlinespace[1pt]
\rowcolor{SparkDiffusionLight}
\multicolumn{2}{l}{
\textcolor{SparkDiffusionRed}{\textbf{Trajectory-Mixed Distillation}}
}
\\


Noise split
& CrossDistill crosspoint ($0.934$) \\

High-noise objective
& PCM-style consistency on the high-noise segment \\

Low-noise objective
& DMD-style distribution matching on the low-noise segment \\

Student sampling steps
& 3 (one high-noise PCM step and two low-noise DMD steps; Sec.~\ref{subsec:distill}) \\

Fake model
& Dense Wan2.1-T2V-14B \\

Fake update ratio
& $10:1$ \\

Student trainable modules
& Full DiT backbone (including our low-rank compensation modules and gating parameters) \\

\addlinespace[1pt]
\rowcolor{SparkDiffusionLight}
\multicolumn{2}{l}{
\textcolor{SparkDiffusionRed}{\textbf{Optimization}}
}
\\

Optimizer
& AdamW \\

Learning rate
& $2\times10^{-6}$ (student);
$4\times10^{-7}$ (fake) \\

Weight decay
& 0.01 \\


Training steps
& 8{,}000  \\

Global batch size
&  16 \\


Training precision
& BF16 \\


EMA
& Disabled \\


\addlinespace[1pt]
\rowcolor{SparkDiffusionLight}
\multicolumn{2}{l}{
\textcolor{SparkDiffusionRed}{\textbf{Compute}}
}
\\

Hardware
& $8\times$ NVIDIA H100 80GB\\

Wall-clock training time
& $\approx$130 hours \\

GPU Hours
& $\approx$1\,K GPU hours \\

\bottomrule[1.25pt]

\end{tabular*}

\arrayrulecolor{black}
\end{table}


\end{document}